\documentclass{article}

\PassOptionsToPackage{numbers, sort&compress}{natbib}

   \usepackage[preprint]{neurips_2024}

\usepackage{subcaption,graphicx}
\usepackage[utf8]{inputenc} 
\usepackage[T1]{fontenc}    
\usepackage{hyperref}       
\usepackage{url}            
\usepackage{booktabs}       
\usepackage{amsfonts}       
\usepackage{nicefrac}       
\usepackage{microtype}      
\usepackage{xcolor,amsmath,amsthm,amssymb,euscript,dsfont,mathrsfs,enumerate,algorithm,algorithmic}         

\title{Long-time asymptotics of noisy SVGD \\ outside the population limit}
\author{V.~Priser \\
Télécom Paris\\
\And P.~Bianchi\\
Télécom Paris\\
\And A.~Salim\\
Microsoft Research\\
}
\newcommand{\eqdef}{:=}

\newcommand{\ps}[1]{ \langle#1 \rangle}
\newcommand{\bP}{{\mathbb P}}
\newcommand{\bE}{{\mathbb E}}
\newcommand{\bR}{{\mathbb R}}
\newcommand{\bI}{{\mathbb I}}
\newcommand{\bN}{{\mathbb N}}
\newcommand{\cA}{{\mathcal A}}

\newcommand{\cB}{{\mathcal B}}

\newcommand{\cC}{{\mathcal C}}
\newcommand{\cD}{{\mathcal D}}
\newcommand{\cL}{{\mathcal L}}

\newcommand{\cP}{{\mathcal P}}

\newcommand{\cE}{{\mathcal E}}
\newcommand{\cH}{{\mathcal H}}

\newcommand{\mcL}{{\mathscr L}}
\newcommand{\cF}{{\mathcal F}}
\newcommand{\cK}{{\mathcal K}}

\newcommand{\sV}{{\mathsf V}} 

\newcommand{\sW}{{\mathsf W}}

\newcommand{\dr}{{d}}
\newcommand{\ex}{\mathrm{e}}

\newcommand{\6}{\partial}

\newcommand\KL[2]{\mathrm{KL}(#1||#2)}

\newcommand\IS[2]{\mathcal{I}_{\text{stein}}(#1||#2)}
\newcommand\I[2]{\mathcal{I}(#1||#2)}
\newcommand\esp[1]{\mathbb{E}\left[#1\right]} 
\newcommand{\indicatrice}{\mathds{1}}

\newcommand\dd[2]{\frac{{d} #1}{{d} #2}}

\newcommand\p[1]{\left( #1 \right)}

\newcommand\norm[1]{\left\lVert #1 \right\rVert}
\newcommand\abs[1]{\left\lvert #1 \right\rvert}

\theoremstyle{plain}

\newtheorem{theorem}{Theorem}
\newtheorem{coro}{Corollary}
\newtheorem{lemma}{Lemma}
\newtheorem{prop}{Proposition}
\newtheorem{assumption}{Assumption}
\theoremstyle{remark}

\newtheorem{definition}{Definition}
\begin{document}

\maketitle

\begin{abstract}
  Stein Variational Gradient Descent (SVGD) is a widely used sampling algorithm that has been successfully applied in several areas of Machine Learning. SVGD operates by iteratively moving a set of $n$ interacting particles (which represent the samples) to approximate the target distribution. Despite recent studies on the complexity of SVGD and its variants, their long-time asymptotic behavior (i.e., after numerous iterations $k$) is still not understood in the finite number of particles regime. We study the long-time asymptotic behavior of a noisy variant of SVGD. First, we establish that the limit set of noisy SVGD for large $k$ is well-defined. We then characterize this limit set, showing that it approaches the target distribution as $n$ increases. In particular, noisy SVGD provably avoids the variance collapse observed for SVGD. Our approach involves demonstrating that the trajectories of noisy SVGD closely resemble those described by a McKean-Vlasov process.
\end{abstract}

\section{Introduction}

Sampling is a fundamental task of machine learning, at the core of Bayesian inference and generative modeling. Mathematically, the task of sampling can be formulated as the task of generating samples, \textit{i.e.}, random variables, from a given (or learnt) probability distribution $\pi$. This task can be achieved by means of a sampling algorithm that iteratively generates the samples, which are meant to asymptotically approximate the target distribution.

The question of the convergence in distribution of the samples to the target $\pi$ is therefore of primary interest in the theory of sampling. This question has been investigated by several works in the sampling literature, with precise convergence rates for some sampling algorithms such as the celebrated Langevin algorithm, see~\cite{chewi2023log} for an overview. 

Stein Variational Gradient Descent (SVGD)~\cite{liu2016stein} is an algorithm to sample from a target distribution $\pi$ whose density w.r.t. Lebesgue measure is known up to a normalizing factor and written in the form
\begin{equation}
\pi(x) \propto \exp(-F(x)), \quad \text{where} \quad F: \bR^d \to \bR.
\end{equation}
SVGD (and its variants) is an alternative to the Langevin algorithm that has been successfully applied in several areas of machine learning, see~\cite{liu2017policy,zhang2018learning,zhang2019scalable,tao2019variational,pu2017vae,kassab2020federated,messaoud2024s} among others. For example, the SVGD dynamics can be seen as a "kernelized" version of the probability flow ODE used in generative modeling~\cite{song2020score,chen2024probability}. The SVGD algorithm takes the form of an interacting particles system of $n$ particles. The empirical distribution of the $n$ particles at time $k$, denoted $\mu^n_k$, is meant to approximate the target $\pi$ when the number of iterations $k$ is large.




\subsection{Related works}
Several works have investigated the convergence of SVGD, \textit{i.e.}, the convergence of $\mu^n_k$ to $\pi$. 

Most of these works have considered the hypothetical regime $n = \infty$, called the population limit~\cite{korba2020non,salim2022convergence,sun2023convergence,nusken2021stein}. More precisely, in the population limit, \cite{korba2020non,salim2022convergence,sun2023convergence} showed that for every $k>0$, 
\begin{equation}
\label{eq:poplimit}
\IS{\mu^\infty_k}{\pi} < \frac{C}{k},
\end{equation}
where $C>0$ is a constant and $\mathcal{I}_{\text{stein}}$ denotes the Stein Fisher Information, a discrepancy between the current iterate $\mu^\infty_k$ and the target $\pi$. The convergence in distribution of SVGD to the target $\pi$ can be deduced, 
in the population limit, by letting $k \to \infty$ in~\eqref{eq:poplimit}, see~\cite{salim2022convergence}.

More recently, some works have considered the finite number of particles regime $n < \infty$~\cite{shi2024finite,das2024provably,carrillo2023convergence,liu2024towards,jaghargh2023stochastic}. 
More precisely, in this regime, one can show that SVGD approximates its population limit provided that $k$ is small enough~\cite{korba2020non,shi2024finite,lu2019scaling,liu2017stein}. Combining this fact with~\eqref{eq:poplimit}, \cite{shi2024finite,carrillo2023convergence} showed that $\IS{\mu^n_k}{\pi} < C'/k$, where $C'>0$ is a constant, provided that $k$ is small enough (\textit{e.g.}, $k < \log \log (n)$ in~\cite{shi2024finite}). Because of this upper bound on $k$, the convergence of SVGD, in the finite number of particles regime, cannot be deduced by letting $k \to \infty$. 

Indeed, SVGD does not converge to the target when $n < \infty$. Because the iterates of SVGD are discrete measures with a finite support of $n$ points, whereas the target $\pi$ has a continuous density w.r.t. Lebesgue. Therefore, we ask the following question.

\emph{What does SVGD converge to (\textit{i.e.}, when $k \to \infty$) in the finite number of particles regime (\textit{i.e.}, when $n < \infty$ is fixed)?}

To the best of our knowledge, this question remains unanswered except in the particular case where $\pi$ is a centered Gaussian distribution, see~\cite[Theorem 10]{liu2024towards}. For a fixed \( n \), the paper \cite{jaghargh2023stochastic} demonstrates that SVGD converges in expectation to a system of \( n \) continuous-time particles, but does not enable the establishment of consistency with the target distribution \( \pi \), when \( n \) becomes large.

However, we can already make a few observations. 

\begin{itemize}
\item As mentioned above, SVGD does not converge to the target $\pi$ because the iterates of SVGD are discrete whereas $\pi$ is continuous.  
\item The best one can hope in general is for the SVGD iterates to converge to some "limit" $\mcL^n$ that approaches $\pi$ as $n$ grows.
\item Even if we were able to show that the limit $\mcL^n$ is well-defined (this task is already non trivial since some particles could diverge for example), $\mcL^n$ would probably not approach the target $\pi$ as $n$ grows. Indeed, SVGD has been empirically shown not to converge to the target $\pi$ in high dimension. More precisely, SVGD has been observed to underestimate the variance of the target distribution and the particles of SVGD have been observed to collapse to some modes of the distribution, see~\cite{ba2021understanding,zhuo2018message,d2021annealed}. 
\end{itemize}

\subsection{Contributions}

In this paper, we introduce a new noisy variant of SVGD where each iteration is regularized by noise which takes the form of an iteration of the Langevin algorithm. We study the "limit" $\mcL^n$ of our algorithm, noisy SVGD, with $n < \infty$ particles, when the number of iterations $k \to \infty$. More precisely, our contributions are the following.

\begin{itemize}
\item We propose a new noisy variant of SVGD where each iteration is regularized by noise which takes the form of an iteration of the Langevin algorithm.
\item We first show that, when the number of particles $n < \infty$ is fixed, noisy SVGD converges when $k \to \infty$ to a well-defined limit set $\mcL^n$ (Th.~\ref{th:tight}).
\item Then, we describe this limit set $\mcL^n$: it cannot contain the target $\pi$, but we show that $\mcL^n$  approaches $\pi$ as $n$ grows (Th.~\ref{th:ergo}).
\item Finally, we obtain Cor.~\ref{coro:ergodic} on the convergence of noisy SVGD in the regime $\displaystyle \lim_{n \to \infty} \lim_{k \to \infty}$. Since the convergence in the regime $\displaystyle \lim_{k \to \infty} \lim_{n \to \infty}$ can be deduced from the existing works mentioned above, Cor.~\ref{coro:ergodic} implies that $\displaystyle \lim_{n \to \infty}$ and $\displaystyle \lim_{k \to \infty}$ can be exchanged. 
\item Our overall approach relies on proving that the trajectories of noisy SVGD mimic that of a McKean-Vlasov process~\cite{bianchi2024long}, a dynamical result of independent interest (Proposition~\ref{prop:MinfV2}).
\item Our convergence results prove that noisy SVGD avoids the variance collapse of SVGD, a fact that we verify experimentally by comparing noisy SVGD to SVGD (Fig.~\ref{fig:images}). 
\end{itemize}

\subsection{Paper structure}
This paper is organized as follows. We review some background material in Section~\ref{sec:background}. In Section~\ref{sec:svgd}, we introduce our main algorithm, noisy SVGD. Next, we state our main results regarding the convergence of noisy SVGD in Section~\ref{sec:marginal}. In Section~\ref{sec:proof}, we provide an overview of our convergence proof, which relies on relating the trajectories of noisy SVGD with those of a McKean-Vlasov process. In Section~\ref{sec:collapse}, we empirically show that noisy SVGD, unlike SVGD, does not suffer from the particles collapse. Finally, we conclude in Section~\ref{sec:ccl}. The proofs are deferred to the Appendix.

\section{Background }
\label{sec:background}
\subsection{Notations}


The Euclidean inner product and norm of $\bR^d$ are denoted $\langle \cdot, \cdot \rangle$ and $\| \cdot \|$. We consider a Reproducing Kernel Hilbert Space (RKHS) $\cH_0$ whose kernel is denoted $K: \bR^d\times \bR^d \to \bR$. The product space $\cH:= \cH_0^d$, is a Hilbert space whose inner product and norm are denoted $\langle \cdot, \cdot \rangle_{\cH}$ and $\| \cdot \|_{\cH}$.



\subsection{Optimal transport}
For every topological space $E$, we denote by $\cP(E)$ the set of probability measures on
the Borel $\sigma$-field $\cB(E)$.
If $E$ is a Polish (complete, metrizable) space, then $\cP(E)$ equipped with the weak$\star$ topology is Polish as well.
A subset $\cA$ of random variables on $E$ is called \emph{tight}, if, for every $\varepsilon>0$, there exists a compact set $A\subset E$, such that $\bP(X\in A)>1-\varepsilon$, for every $X\in \cA$.
If $E$ is a Banach space, we define
$$
\cP_2(E)\eqdef \{\mu\in \cP(E)\,:\, \int\|x\|^2\dr\mu(x)<\infty\}\,,
$$
and the Wasserstein-2 distance by
\begin{equation*}\label{eq:WpE}
{ W}_2(\mu,\nu) \eqdef \left(\inf_{\varsigma\in \Pi(\mu,\nu)} 
  \int \|x-y\|^2 d\varsigma(x,y)\right)^{1/2}\,,
\end{equation*}
where $\Pi(\mu,\nu)$ is the set couplings of $\mu \in \cP_2(E)$ and $\nu \in \cP_2(E)$, \textit{i.e.}, the set of measures $\varsigma\in\cP(E\times E)$ such 
that $\varsigma(\,\cdot\,\times E)=\mu$ and $\varsigma(E\times\,\cdot\,)=\nu$.
The Wasserstein space, \textit{i.e.}, the set $\cP_2(E)$ endowed with the distance ${ W}_2$, is a Polish space. 



In the proofs, we need to consider the case where the space $E$ coincides with the set $\cC$ of continuous function on $[0,\infty)$ to $\bR^d$. Eventhough $\cC$ is not a Banach space, the definitions follow the same lines.
The set $\cC$ is equipped with the topology of uniform convergence on compact intervals.
For every $\rho\in \cP(\cC)$, we denote by $\rho^T$ the restriction of $\rho$ to functions on the compact interval $[0,T]$ (that is, $\rho^T = (\pi_{[0,T]})_\#\rho$, the pushforward of $\rho$ by the map $\pi_{[0,T]}$ which, to every function $f\in \cC$, associates its restriction to the compact interval $[0,T]$). We denote by $\cP_2(\cC)$ the set of measures $\rho\in \cP(\cC)$ such that $\rho^T\in \cP_2(C([0,T],\bR^d))$ for all $T>0$. This space is naturally equipped with the following topology: a sequence $\rho_n$ converges to $\rho$ in the Wasserstein-2 sense if $\rho_n^T\to \rho^T$ in the 
Wasserstein-2 sense, for every $T>0$. Then, $\cP_2(\cC)$ is metrizable, and we denote by $W_2(\rho,\rho')$ a proper distance~\cite[Sec. 2.2]{bianchi2024long}.


\subsection{Functional inequalities  }
Let $\pi \in \cP_2(\bR^d)$ be the target distribution, \textit{i.e.}, $\pi \propto \exp(-F)$.
The Kullback-Leibler divergence with respect to $\pi$ is defined for every $\mu \in \cP_2(\bR^d)$ by
\[
\KL{\mu}{\pi} = \int {\log\dd{\mu}{\pi}}\dr\mu\,,
\]
if $\mu$ has a density $\dd{\mu}{\pi}$ w.r.t. $\pi$, and $\KL{\mu}{\pi} = +\infty$ else.
The Stein Fisher Information w.r.t. $\pi$ is defined by 
\[
  \IS{\mu}{\pi} := \norm{P_\mu\nabla\log \dd{\mu}{\pi}}^2_\cH,
\]
where $P_\mu : L^2(\mu) \to \cH$ is the so-called kernel integral operator $P_\mu f = \int K(\cdot,y)f(y)\dr\mu(y)$.
The Fisher Information w.r.t. $\pi$ is defined by
\[
\I{\mu}{\pi} := \int \norm{\nabla\log \dd{\mu}{\pi}}^2 \dr \mu(x)\,.
\]
Finally, we recall the Log Sobolev Inequality (LSI) that relates the Kullback-Leibler divergence and the Fisher Information.
\begin{definition}[Logarithmic Sobolev Inequality]
   The distribution $\pi$ satisfies the Logarithmic Sobolev Inequality, if there exists $\alpha>0$ such that for every $\mu\in\cP_2(\bR^d)$, 
   \[
    \KL{\mu}{\pi} \le \frac 1{2\alpha} \I{\mu}{\pi}.   
   \]
\end{definition}
The LSI is satisfied when $F$ is $\alpha$-strongly convex but can also be used to study the convergence of sampling algorithms in the case where $F$ is not convex~\cite[Section 21]{villani2008optimal} (see also~\cite{vempala2019rapid}).
 
\section{Noisy Stein Variational Gradient Descent}
\label{sec:svgd}

The Stein Variational Gradient Descent (SVGD) algorithm~\cite{liu2016stein} is used to sample from a distribution $\pi \propto \exp(-F)$, where $F:\bR^d\to\bR$ is a differentiable function. At every iteration $k$, the algorithm updates the values of $n$ $\bR^d$-valued vectors, refered to as the particles $X_k^{1,n},\cdots,X_k^{n,n}$.
We study a generalization of SVGD, called noisy SVGD, that incorporates noise in the form of a Langevin iteration at each step of SVGD.

Let $(\Omega,\cF,\bP)$ be a probability space, $\lambda\ge 0$ and $(\gamma_k)$ be a positive deterministic sequence in $\bR$. 
Starting with a $n$--uple $(X_0^{1,n},\dots, X_0^{n,n})$ of $\bR^d$-valued random variables, the particles are updated according to Algorithm~\ref{tab:SVGD} where $(\xi^{i,n}_k)_{i,k}$ is a family of i.i.d standard Gaussian vectors in $\bR^d$.
\begin{algorithm}[H]
    \caption{Noisy Stein Variational Gradient Descent}
    \begin{algorithmic}
        \STATE {\bf Initialization}: generate $n$ particles $(X_0^{1,n},\dots, X_0^{n,n})$  \FOR{$k=0,1,2,\ldots$}
        \FOR{$i=1,2,\ldots,n$}
        \STATE \begin{multline}\label{eq:algo}
  X_{k+1}^{i,n} = X_k^{i,n} - \frac{\gamma_{k+1}}{n}\sum_{j\in [n]} \p{K(X^{i,n}_k,X^{j,n}_k) \nabla F(X^{j,n}_k)-\nabla_2 K(X^{i,n}_k,X^{j,n}_k)  }\\  \underbrace{- \lambda\gamma_{k+1}\nabla F(X^{i,n}_k) +\sqrt{2\lambda \gamma_{k+1}} \xi^{i,n}_{k+1}}_{\text{Langevin regularization}}\,.
\end{multline}
        \ENDFOR
        \ENDFOR
    \end{algorithmic}
    \label{tab:SVGD}
\end{algorithm}
Noisy SVGD boils down to the standard deterministic SVGD algorithm when $\lambda = 0$. The regularization parameter \(\lambda>0\) allows the introduction of noise into the algorithm with the aim of preventing the mode collapse phenomenon described in the introduction.
We state our assumptions on the step size and the noise sequence.
\begin{assumption}\label{hyp:algo}
    Let the following holds.
\begin{enumerate}[i)]
    \item   \label{hyp:gamma}  $(\gamma_k)$ is a non-negative deterministic sequence satisfying 
    $\lim_{k\to\infty}\gamma_k = 0$, and  $\sum_k\gamma_k = +\infty$.
    \item  \label{hypenum:exch}$(\xi^{i,n}_{k})_{k\in\bN, i\in[n]}$ is an i.i.d. sequence of standard Gaussian variables, independent of $(X_0^{i,n})_{i\in [n]}$.
\end{enumerate}
\end{assumption}
Noisy SVGD allows for the approximation of linear functionals of the form \(\int f \, d\pi\), where \(f\) is an arbitrary integrand, by the discrete sum
$\frac{1}{n} \sum_{i=1}^n f(X_k^{i,n})\,.$
The latter can be written as \(\int f \, d\mu_k^n\), where \(\mu_k^n\) is the empirical measure of the particles, defined by
$$
\mu_k^n \eqdef \frac{1}{n} \sum_{i\in[n]} \delta_{X_k^{i,n}}\,.
$$
Note that \((\mu_k^n)_k\) is a sequence of \emph{random} measures. A useful convergence result for noisy SVGD involves studying the convergence in probability of this sequence towards the target distribution \(\pi\).
In some situations, it is more convenient to study the \emph{averaged} empirical measure $\bar\mu_k^n$, defined for $k,n\in\bN^*$, by:
\[
    \bar\mu_k^n := \frac{\sum_{i\in[k]}\gamma_i \mu_i^n}{\sum_{i\in[k]}\gamma_i }\,.
\]


\section{Convergence results of noisy SVGD}\label{sec:marginal}


\subsection{Limit set of noisy SVGD is well-defined}

We start our analysis by studying the limit set of SVGD as $k$ tend to infinity, for a fixed number $n$ of particles. As the number of particles is fixed, it cannot be expected that the limit of $\mu_k^n$ coincides with $\pi$ as $k\to\infty$, because a discrete measure with a fixed number of atoms cannot approach a density. 
We formally describe the limit set of the empirical measures in a distributional sense 
\begin{definition}[Distributional limit set]
    Let $\nu$, $(\nu_k:k\in\bN)$ be random variables on $\cP(\bR^d)$. We say that $\nu$ is a distributional cluster point of $(\nu_k)$, if $\nu_k$ converges in distribution to $\nu$ along a subsequence. The distributional limit set $\mathscr{L}((\nu_k))$ of the sequence $(\nu_k)$ is defined as the set of distributional cluster points of $(\nu_k)$.
\end{definition}
We denote by $\mathscr{L}^n\eqdef \mathscr{L}((\mu_k^n))$ the distributional limit set of the sequence $(\mu_k^n:k\in\bN)$, when $k\to\infty$, $n$ being fixed. In words, $\mathscr{L}^n$ is the set of random measures $\nu^n$ such that $\mu_k^n$ converges to $\nu_n$ in distribution, along a subsequence. 
Similarly, we denote by $\bar{\mathscr{L}}^n$ the limit set of the sequence $(\bar\mu_k^n)$.
\begin{assumption}\label{hyp:stab}
There exists four non-negative constant $c,c',C,C'$, such that for every $x,y\in\bR^d$, the following holds.
\begin{enumerate}[i)]
    \item   The hessian $H_F(x)$ is well-defined and $\norm{H_F(x)}_{op}\le C$.

    \item  \label{hypenum:dissipation}$c' F(x) -C\le \norm{\nabla F(x)}^2\le C' F(x) + C$ and  $ c\norm{x}^2-  C  \le F(x)$.
    \item  $  \norm{K(\cdot,y)}_{\cH_0} + \norm{\nabla_2 K(\cdot, y)}_{\cH}\le C .$
        \item \label{hypenum:}$\sup_{n}\bE\p{(X_0^{1,n})^4} <\infty .$
\end{enumerate}
\end{assumption}
Given the previous assumption, we can establish the stability of our algorithm, in the form of the following lemma.
\begin{lemma}\label{lem:stab}
  Let Assumptions~\ref{hyp:algo} and~\ref{hyp:stab} be satisfied. Assume $\lambda >0$. Then,
  $
    \sup_{k,n} \bE\|  X^{1,n}_k \|^4 <\infty 
  $.
\end{lemma}
Lem.~\ref{lem:stab} is the key component for establishing our first theorem.
\begin{theorem}\label{th:tight}
Let Assumptions \ref{hyp:algo} and~\ref{hyp:stab} hold. Assume $\lambda>0$. Then, for every $n\in\bN^*$, the sequence of random variables $(\mu_k^n)_k$ is tight.
As a consequence, the sets $\mathscr{L}^n$ and $\bar{\mathscr{L}}^n$ are non empty.
Finally, all random measures of $\mathscr{L}^n$ and $\bar{\mathscr{L}}^n$ belong almost surely to $\cP_2(\bR^d)$.
\end{theorem}
It remains to characterize the limit sets. As mentioned earlier, the random variable equal to $\pi$ a.s. does not belong to the set $\mathscr{L}^n$. 
Therefore, the question is whether $\mathscr{L}^n$ reduces to the singleton $\pi$ as $n$ goes to infinity.

\subsection{Description of the limit set}

Consider the target measure $\pi$.
\begin{definition}
For every $n\geq 1$, let $\mathscr{E}^n$ be a set of random measures on $\cP_2(\bR^d)$. We say that the sequence of random sets $(\mathscr{E}^n:n\in \bN^*)$ converges in probability to $\pi$, denoted by $\mathscr{E}^n\xrightarrow[]{\bP}\pi$, if the Hausdorff-Wasserstein distance between $\mathscr E^n$ and $\pi$ converges in probability to zero:
$$
\forall \varepsilon>0,\ \lim_{n\to\infty} \bP(\sup_{\nu\in \cE^n} W_2(\nu,\pi)>\varepsilon) = 0\,.
$$
\end{definition}

Consider the following regularity assumption on the kernel $K$. 
\begin{assumption}\label{hyp:holder}
    There exists $\beta>0$, such that for every $x,x',y\in\bR^d$, we obtain $$\abs{K(x,y) -K(x',y)} + \norm{\nabla_2 K(x,y) - \nabla_2 K(x',y)}\le C \norm{x-x'}^\beta\,.$$
\end{assumption}
\begin{theorem}\label{th:ergo}
    Let Assumptions~\ref{hyp:algo}
,~\ref{hyp:stab}, and~\ref{hyp:holder} hold. Assume $\lambda>0$. 
Then,
$$
\bar{\mathscr{L}}^n\xrightarrow[n\to\infty]{\bP}\pi\,.
$$
\end{theorem}
The motivation for studying the limit set \(\bar{\mathscr{L}}^n\) of the \emph{averaged} measure \(\bar\mu_k^n\) is technical. 
However, the limit set \(\mathscr{L}^n\) of the (non-averaged) empirical measure \(\mu_k^n\) can also be characterized, provided an additional assumption on the target density is met.
\begin{assumption}\label{hyp:LSI}
    The distribution $\pi$ satisfies the Logarithmic Sobolev Inequality for a constant $\alpha >0$.
\end{assumption}
\begin{theorem}\label{th:nonergo}
    Let Assumptions~\ref{hyp:algo}
,~\ref{hyp:stab},~\ref{hyp:holder} and~\ref{hyp:LSI} hold. Assume $\lambda>0$. Then,
$$
{\mathscr{L}}^n\xrightarrow[n\to\infty]{\bP}\pi\,.
$$
\end{theorem}

\subsection{Long-time convergence of the empirical measure}

As a consequence of Th.~\ref{th:ergo} and Th.~\ref{th:nonergo} respectively, we can characterize the long-time convergence of the empirical measure of the particles, averaged and non-averaged respectively.
\begin{coro}\label{coro:ergodic}
    Let Assumptions~\ref{hyp:algo}
,~\ref{hyp:stab} and~\ref{hyp:holder} hold. Assume $\lambda>0$. Then, for every $\varepsilon>0$,
\[
\lim_{n\to\infty}\limsup_{k\to\infty} \bP(W_2(\bar \mu_k^n,\pi) >\varepsilon) = 0\,.
\]
If Assumption~\ref{hyp:LSI} moreover holds, the same result holds when $\bar\mu_k^n$ is replaced by $\mu_k^n$.
\end{coro}
 Since the convergence in the regime $\displaystyle \lim_{k \to \infty} \limsup_{n \to \infty}$ can be deduced from the existing works mentioned above, Cor.~\ref{coro:ergodic} implies that $\displaystyle \lim_{n \to \infty}$ and $\displaystyle \lim_{k \to \infty}$ can be exchanged.

\section{Overview of the convergence proof and dynamical behavior of noisy SVGD}
\label{sec:proof}

The method used to prove our main result involves studying the convergence of the particles at the level of stochastic processes.

\subsection{Interpolated process}
We consider for each $i\in[n]$ the random 
continuous-time process
$\bar X^{i,n} : [0,\infty) \to \bR^d, t\mapsto \bar X_t^{i,n}$ defined as the
piecewise linear interpolation of the particles $(X_k^{i,n})_k$. Specifically,
writing 
$
\tau_k\eqdef \sum_{j=1}^k\gamma_j 
$
, for each $k\in\bN$, we define: 
\begin{equation*}
\label{eq:interp} 
\forall t \in [\tau_k, \tau_{k+1}), \quad 
\bar X^{i,n}_t := 
X_k^{i,n} + \frac{t - \tau_k}{\gamma_{k+1}} 
 \left( X_{k+1}^{i,n}-X_k^{i,n} \right) . 
\end{equation*} 
The interpolated processes $\bar X^{i,n}$, for $i \in [n]$, are elements of the set
$\cC$ of continuous functions on $[0,\infty)\to\bR^d$.
Rather than solely examining the empirical measure of the particles $X_k^{i,n}$, our approach focuses on analyzing the empirical measure of the interpolated processes $\bar X^{i,n}$ across the entire positive real line. Define:
\begin{equation*}
  \label{eq:m}
  m^n_t := \frac 1n\sum_{i=1}^n \delta_{\bar X^{i,n}_{t+\cdot}}\,,
\end{equation*}
for each \(n\) and \(t\). Note that $m_t^n$ is a random variable on $\cP_2(\cC))$. The empirical measure $\mu_k^n$ of the discrete particles can be deduced from $m_t^n$ by marginalization, which is why we focus on $m_t^n$ from now on.

\subsection{McKean-Vlasov distributions}

For a fixed $n$, the particles $X_k^{i,n}$, for $i\in [n]$, can be interpreted as an Euler discretization scheme of a stochastic differential equation involving $n$ continuous-time particles. As the discretization step $\gamma_k$ tends to zero, the interpolated processes eventually share the same behavior as the continuous-time particles as $k$ tends to infinity.
Moreover, in the population limit where $n$ is large, any of the continuous-time particles coincides, in law, with the solution to a McKean-Vlasov equation, as defined below. This phenomenon is known as the propagation of chaos. We refer to \cite{Chaintron_2022} for a detailed exposition.
\begin{definition}
  \label{def:V}
   We say that a measure $\rho\in\cP_2(\cC)$ is a McKean-Vlasov distribution, if it coincides with the pathwise law of a weak solution \((X_t)_{t\geq 0}\) to the nonlinear Stochastic Differential Equation (SDE)
   \[
dX_t = -\int \left( K(X_t,y) \nabla F(y) - \nabla_2 K(X_t,y) \right) d\rho_t(y) \, dt - \lambda \nabla F(X_t) \, dt + \sqrt{2\lambda} \, dW_t,
\]
where $(W_t)_{t\ge 0}$ is a standard Brownian motion.
Denote by $\sV_2$ the set of  McKean-Vlasov distributions.
\end{definition}

\subsection{Limit measures of noisy SVGD are McKean-Vlasov distributions}

It remains to explain in which sense, the empirical measures $m_t^n$ converge to a McKean-Vlasov distribution as $(t,n)\to(\infty,\infty)$. The question requires the introduction of the following measure:
$$
M_t^n\eqdef\frac 1t\int_0^t \delta_{m^n_s} \dr s\,.
$$
To summarize, we introduced the following of random variables: (process level) $\bar X^{i,n}$ is a r.v. on $\cC$; (process-measure level) $m_t^n$ is a r.v. on $\cP_2(\cC)$; (process-measure-measure level) $M_t^n$ is a r.v. on $\cP(\cP_2(\cC))$.
As a consequence of Lem.~\ref{lem:stab}, we obtain the following result.
\begin{prop}\label{prop:tightm}
   Let Assumptions~\ref{hyp:algo} and~\ref{hyp:stab} be satisfied. Assume $\lambda >0$. For every $n\in\bN^*$, the sequence of random variables $(M^n_t)_t$ is tight.
\end{prop}
In particular, Proposition~\ref{prop:tightm} implies Th.~\ref{th:tight} and the fact that the limit set of SVGD is non-empty. It remains to characterize the latter in the doubly asymptotic regime where $t,n$ both tend to infinity. To that end, we study the (distributional) limit points of $(M_t^n)$, as $(t,n)\to (\infty,\infty)$.
The following result is a extracted from~\cite[Lem.~9]{bianchi2024long}. \begin{prop}\label{prop:MinfV2}
   Let Assumptions~\ref{hyp:algo} and~\ref{hyp:stab} be satisfied. Assume $\lambda >0$. Let $M$ be a random measure on $\cP(\cP_2(\cC))$ such that $M_t^n$ converges in distribution to $M$ as $(t,n)\to (\infty,\infty)$, along some subsequence. Then, $M(\sV_2)=1$ a.s.
\end{prop}
Let us explain the main consequence of this result. Let $f$ be the function defined by $f(\rho)=W_2(\rho,\sV_2)$ for every $\rho\in \cP_2(\bR^d)$. When $M_t^n$ tends to $M$ in distribution along some subsequence, our definition of $M_t^n$ implies that:
$$
\int f dM_t^n = \frac 1t\int_0^t W_2(m_s^n,\sV_2) ds\xrightarrow[]{\mathcal D} \int W_2(\rho,\sV_2) dM(\rho)=0\,,
$$
where the symbol $\xrightarrow[]{\mathcal D}$ stand for convergence in distribution. This shows that, in an ergodic sense, $m_t^n$ converges in probability to the set of McKean-Vlasov distributions, as $(t,n)\to (\infty,\infty)$. 

\subsection{Limit measures of noisy SVGD are time-shift recurrent}

More can be said about the particular McKean-Vlasov distribution in the limit set. For every $\tau>0$, denote by $\Phi_\tau:\cP(\cC)\to\cP(\cC)$ the map which shifts a process-measure by a time $\tau$, namely, $\Phi_\tau(\rho):f\mapsto \int f(x_{\tau+\cdot})d\rho(x)$. Obviously, $\Phi_\tau(m_t^n) = m_{\tau+t}^n$, which in turn implies that, as $t\to\infty$, for every bounded continuous function $G:\cP(\cC)\to\bR$,
$$
\int G(\Phi_\tau(\rho))dM_t^n(\rho) =\frac 1t\int_0^t G(m_{\tau+s}^n)ds \simeq \frac 1t\int_0^t G(m_{s}^n)ds = \int G(\rho) dM_t^n(\rho)\,,
$$
where the precise statement is found in the supplementary (see also \cite[Lem.~10]{bianchi2024long}).
Passing to the limit, this implies that every distributional limit point $M$ of $M_t^n$ is shift-invariant, in the sense that $\int G\circ \Phi_\tau dM = \int G dM$ a.s., for every bounded continuous $G$ and every $\tau>0$. Therefore, by the Poincar\'e recurrence theorem, $M$ is supported by the set of \emph{recurrent} McKean-Vlasov distributions, that is, the set of measures $\rho\in \sV_2$ for which there exists a sequence $\tau_l\to\infty$, such that $\rho = \lim \Phi_{\tau_l}(\rho)$. 

\subsection{Recurrent McKean-Vlasov distributions coincide with the target}

For any process-measure $\rho\in \cP(\cC)$, we denote by $(\rho_t:t\geq 0)$ its marginals in $\cP(\bR^d)$.
\begin{prop}\label{prop:Lya}
    Let Assumption~\ref{hyp:stab} and~\ref{hyp:holder} hold. Assume $\lambda>0$. Let $t_2>t_1>0$.
   For every $\rho\in\sV_2$ and every $t\in [t_1,t_2]$, $\rho_t$ admits a differentiable density w.r.t. the Lebesgue measure. Moreover,
    \begin{equation}
        \label{eq:KLlyap}
          \KL{\rho_{t_2}}{\pi}-\KL{\rho_{t_1}}{\pi}  
             =  -\int_{t_1}^{t_2}\p{\IS{\rho_t}{\pi} + \lambda \I{\rho_t}{\pi}} \dr t\,.
    \end{equation}
\end{prop}
The above proposition shows that the Kullback-Leibler divergence is a Lyapunov function, in the sense that $\KL{\rho_{t_2}}{\pi}\leq \KL{\rho_{t_1}}{\pi}$. The inequality is strict unless the r.h.s. of (\ref{eq:KLlyap}) is zero, which holds when $\rho_t = \pi$ for almost all $t$. This implies that if $\rho$ is a recurrent McKean-Vlasov distribution, its marginals coincide with $\pi$. Therefore, in an ergodic sense, the marginals of the process-measure $m_t^n$ converges in probability to $\pi$, as $(t,n)\to (\infty,\infty)$ (see Prop.~\ref{prop:ergo} in the Appendix).

The last step is to establish Th.~\ref{th:nonergo} under the additional Assumption~\ref{hyp:LSI}. In other words, one should discard the time-averaging. This can be done in the situation where, as $t\to\infty$, the marginal $\rho_t$ of any McKean-Vlasov distribution $\rho\in \sV_2$ converges to $\pi$ uniformly in the initial point $\rho_0$ in a compact set. This can be established using the LSI, as shown by the following result.


\begin{prop}\label{prop:contraction} Let the assumptions of Prop.~\ref{prop:Lya} hold. Moreover, we assume that Assumption~\ref{hyp:LSI} is satisfied with $\alpha>0$ and $\lambda >0$. For any compact set $\cK\subset \cP_2(\cC)$, 
for every $t_2>t_1>0$, there exists a constant $C_{t_1,\cK}>0$ depending on $t_1$ and $\cK$, such that
  \[
   \sup_{\rho\in \sV_2\cap \cK} W_2(\rho_{t_2},\pi) \le C_{t_1,\cK}\ex^{-\alpha\lambda (t_2-t_1)}\,.
  \]
  \end{prop}

\section{Noisy SVGD avoids the particles collapse}
\label{sec:collapse}

The convergence results above show the convergence of noisy SVGD in a doubly asymptotic regime $(k,n) \to (\infty,\infty)$.
These convergence results could be reproduced for the deterministic SVGD algorithm. However, in the case of SVGD, our approach would show the convergence of SVGD to a set that includes the target $\pi$, but can also include Dirac measures at stationary points of $F$. Indeed, the McKean-Vlasov process of SVGD (\textit{i.e.}, the case $\lambda = 0$) is stationary at $\delta_x$ for any $x \in \bR^d$ such that $\nabla F(x) = 0$ and $\nabla_2 K(x,x) = 0$\footnote{On the contrary, every stationary distribution of the McKean-Vlasov process of noisy SVGD (\textit{i.e.}, the case $\lambda > 0$) must have a density w.r.t. Lebesgue thanks to the noise injection.}.


This observation is inline with empirical results showing that the deterministic SVGD algorithm may not converge in high dimension and instead collapse to some Diracs which represent modes of the target distribution~\cite{ba2021understanding,zhuo2018message,d2021annealed}. On the contrary, we showed (Th.~\ref{th:ergo} and~\ref{th:nonergo}) that noisy SVGD converges to the target and, in particular, does not collapse to Dirac measures. In this section, we illustrate this fact experimentally.

Fig.~\ref{fig:images} (see Appendix for larger figures) reproduces an experiment from~\cite{ba2021understanding} on the variance collapse of SVGD. We added our algorithm, noisy SVGD, to the plot. 

\begin{figure}[ht!]
    \centering
    \begin{subfigure}[b]{0.49\textwidth} 
        \centering
        \includegraphics[width=\textwidth]{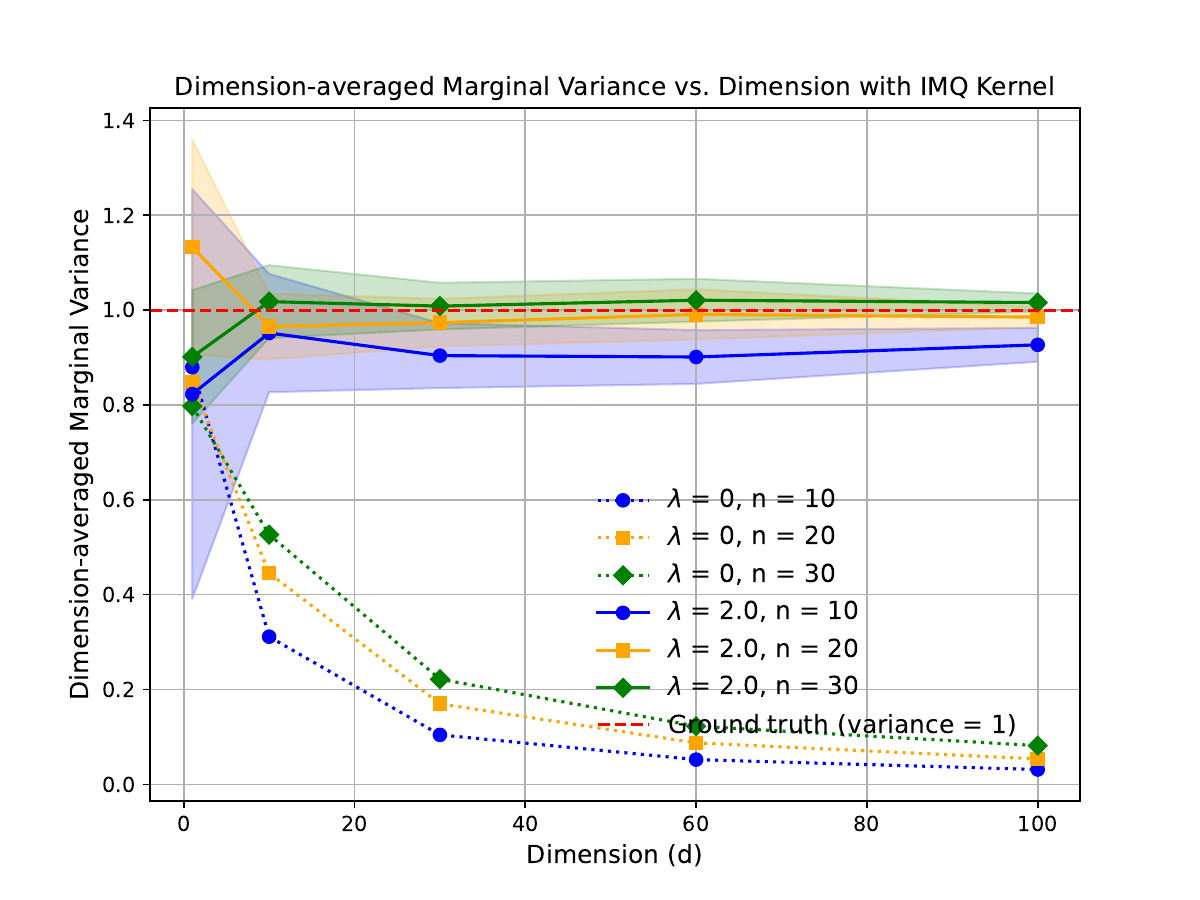} 
        \caption{IMQ kernel}
        \label{fig:image1}
    \end{subfigure}
    \hfill 
    \begin{subfigure}[b]{0.49\textwidth} 
        \centering
        \includegraphics[width=1\textwidth]{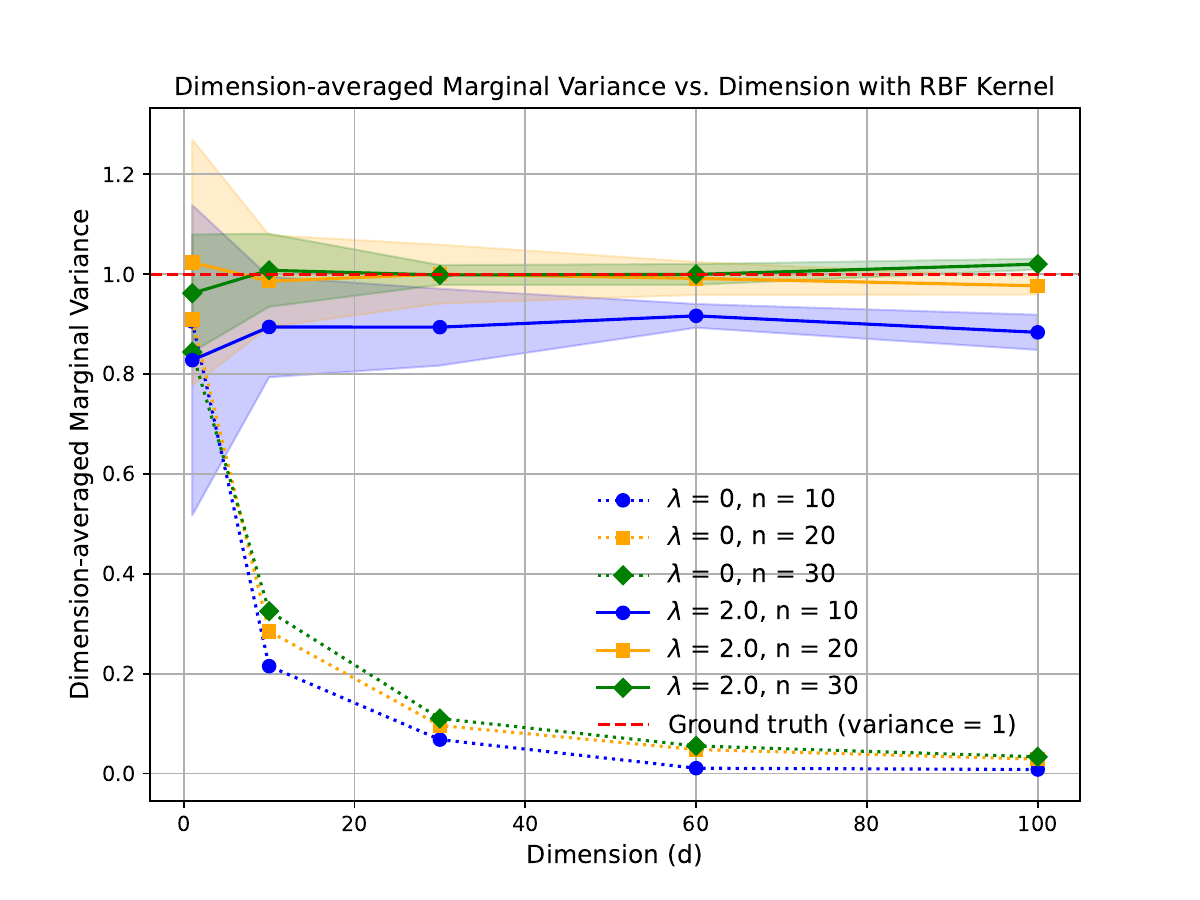} 
        \caption{RBF kernel}
        \label{fig:image2}
    \end{subfigure}
    \caption{Dimension-averaged Marginal Variance of SVGD and noisy SVGD at convergence for sampling from a standard Gaussian.}
    \label{fig:images}
\end{figure}

\textbf{The setup is the following.} We consider the task of sampling from a standard Gaussian with noisy SVGD and SVGD. We use the two most standard kernels for running SVGD: the Radial Basis Function (RBF) kernel, a.k.a. Gaussian kernel $K(x,y) = \exp(-\frac{1}{2}\|x-y\|^2)$ and the Inverse Multi-Quadratic (IMQ) kernel~\cite{gorham2017measuring,kanagawa2022controlling} $K(x,y) = \frac{1}{\sqrt{1+\frac{1}{2}\|x-y\|^2}}$. We simulate noisy SVGD until convergence (\textit{i.e.}, after a large number $k = 200$ of iterations) for different values of the dimension $d$, the number of particles $n$, and the regularization parameter $\lambda$. When $\lambda = 0$, noisy SVGD boils down to the deterministic SVGD. The particles are initialized randomly from a standard Gaussian and the step size is set to $\gamma_k = 10/k$.

Given a probability distribution over $\bR^d$, the Dimension-Averaged Marginal Variance (DAMV) is a statistics of the distribution equal to the average across the $d$ coordinates of the variance of each coordinate. We reproduce an experiment from~\cite{ba2021understanding} where they plotted the DAMV of SVGD after a large number of iterations against the dimension. We added noisy SVGD to the plot, see Fig.~\ref{fig:images}. Since noisy SVGD is random, its DAMV is a random number, therefore we plotted the averaged value of the DAMV over 10 runs and represented the standard deviation of the DAMV in the shaded area behind the curve. Our Python script is available in the Supplementary Material and Fig.~\ref{fig:images} is available in the Appendix in a larger format.   


From Fig.~\ref{fig:images}, two important observations can be made:

\begin{itemize}
\item Since each point in the figure represents a statistical measure (the DAMV) for noisy SVGD after numerous iterations, our theoretical analysis predicts that as $n$ increases, the DAMV values for noisy SVGD should converge to the DAMV of the standard Gaussian, which is $1$. This convergence towards $1$ with increasing $n$ is indeed what we observe in the noisy SVGD data.
\item Contrasting this, SVGD shows a different behavior where its DAMV tends to zero as the dimension increases, as discussed in~\cite{ba2021understanding}. Unlike SVGD, noisy SVGD does not exhibit this variance collapsing behavior.
\end{itemize}

\section{Conclusion}
\label{sec:ccl}
What does a user do? A user sets a finite value for the number $n$ of particles and then runs the algorithm until convergence. Therefore understanding what the algorithm converges to when $n$ is finite is of primary interest. In this work, we provided an understanding of the limit set $\mcL^n$ of noisy SVGD after a large number of iterations. We showed that this limit set is well-defined, and that it approaches the target as $n$ grows. We obtained various conclusions from these results. In particular, noisy SVGD, unlike SVGD, provably avoids collapsing to some modes of the target distribution.

Our work opens the door to several questions regarding the convergence speed of noisy SVGD. First, can we quantify the convergence of noisy SVGD to the set $\mcL^n$? Then, can we quantify the convergence of the set $\mcL^n$ to the target? Finally, how to choose the regularization parameter $\lambda$ and what is its effect on the convergence rate? 

These problems, which are not covered in the literature on SVGD and its variants, would strengthen our understanding of interacting particles systems for sampling, in a regime that matters from a practical perspective.


\bibliographystyle{plain}
\bibliography{bib}

\begin{thebibliography}{10}

\bibitem{ambrosio2005gradient}
L.~Ambrosio, N.~Gigli, and G.~Savar\'{e}.
\newblock {\em Gradient flows in metric spaces and in the space of probability
  measures}.
\newblock Lectures in Mathematics ETH Z\"{u}rich. Birkh\"{a}user Verlag, Basel,
  second edition, 2008.

\bibitem{ba2021understanding}
J.~Ba, M.~A. Erdogdu, M.~Ghassemi, S.~Sun, T.~Suzuki, D.~Wu, and T.~Zhang.
\newblock Understanding the variance collapse of svgd in high dimensions.
\newblock In {\em International Conference on Learning Representations}, 2021.

\bibitem{bianchi2024long}
P.~Bianchi, W.~Hachem, and V.~Priser.
\newblock Long run convergence of discrete-time interacting particle systems of
  the mckean-vlasov type.
\newblock {\em arXiv preprint arXiv:2403.17472}, 2024.

\bibitem{billingsley2013convergence}
P.~Billingsley.
\newblock {\em Convergence of probability measures}.
\newblock Wiley Series in Probability and Statistics: Probability and
  Statistics. John Wiley \& Sons, Inc., New York, second edition, 1999.
\newblock A Wiley-Interscience Publication.

\bibitem{carmeli2010vector}
C.~Carmeli, E.~De~Vito, A.~Toigo, and V.~Umanit{\'a}.
\newblock Vector valued reproducing kernel hilbert spaces and universality.
\newblock {\em Analysis and Applications}, 8(01):19--61, 2010.

\bibitem{carrillo2023convergence}
J.~A. Carrillo and J.~Skrzeczkowski.
\newblock Convergence and stability results for the particle system in the
  stein gradient descent method.
\newblock {\em arXiv preprint arXiv:2312.16344}, 2023.

\bibitem{Chaintron_2022}
L.-P. Chaintron and A.~Diez.
\newblock Propagation of chaos: A review of models, methods and applications.
  i. models and methods.
\newblock {\em Kinetic and Related Models}, 15(6):895, 2022.

\bibitem{chen2024probability}
S.~Chen, S.~Chewi, H.~Lee, Y.~Li, J.~Lu, and A.~Salim.
\newblock The probability flow ode is provably fast.
\newblock {\em Advances in Neural Information Processing Systems}, 36, 2024.

\bibitem{chewi2023log}
S.~Chewi.
\newblock Log-concave sampling.
\newblock {\em Book draft available at https://chewisinho. github. io}, 2023.

\bibitem{d2021annealed}
F.~D'Angelo and V.~Fortuin.
\newblock Annealed stein variational gradient descent.
\newblock {\em arXiv preprint arXiv:2101.09815}, 2021.

\bibitem{das2024provably}
A.~Das and D.~Nagaraj.
\newblock Provably fast finite particle variants of svgd via virtual particle
  stochastic approximation.
\newblock {\em Advances in Neural Information Processing Systems}, 36, 2024.

\bibitem{gorham2017measuring}
J.~Gorham and L.~Mackey.
\newblock Measuring sample quality with kernels.
\newblock In {\em International Conference on Machine Learning}, pages
  1292--1301. PMLR, 2017.

\bibitem{kanagawa2022controlling}
H.~Kanagawa, A.~Barp, A.~Gretton, and L.~Mackey.
\newblock Controlling moments with kernel stein discrepancies.
\newblock {\em arXiv preprint arXiv:2211.05408}, 2022.

\bibitem{jaghargh2023stochastic}
M.~R. Karimi, Y.-P. Hsieh, and A.~Krause.
\newblock Stochastic approximation algorithms for systems of interacting
  particles.
\newblock In A.~Oh, T.~Naumann, A.~Globerson, K.~Saenko, M.~Hardt, and
  S.~Levine, editors, {\em Advances in Neural Information Processing Systems},
  volume~36, pages 55826--55847. Curran Associates, Inc., 2023.

\bibitem{kassab2020federated}
R.~Kassab and O.~Simeone.
\newblock Federated generalized bayesian learning via distributed stein
  variational gradient descent.
\newblock {\em arXiv preprint arXiv:2009.06419}, 2020.

\bibitem{korba2020non}
A.~Korba, A.~Salim, M.~Arbel, G.~Luise, and A.~Gretton.
\newblock A non-asymptotic analysis for stein variational gradient descent.
\newblock {\em Advances in Neural Information Processing Systems},
  33:4672--4682, 2020.

\bibitem{liu2017stein}
Q.~Liu.
\newblock Stein variational gradient descent as gradient flow.
\newblock {\em Advances in neural information processing systems}, 30, 2017.

\bibitem{liu2016stein}
Q.~Liu and D.~Wang.
\newblock Stein variational gradient descent: A general purpose bayesian
  inference algorithm.
\newblock {\em Advances in Neural Information Processing Systems}, 2016.

\bibitem{liu2024towards}
T.~Liu, P.~Ghosal, K.~Balasubramanian, and N.~Pillai.
\newblock Towards understanding the dynamics of gaussian-stein variational
  gradient descent.
\newblock {\em Advances in Neural Information Processing Systems}, 36, 2024.

\bibitem{liu2017policy}
Y.~Liu, P.~Ramachandran, Q.~Liu, and J.~Peng.
\newblock Stein variational policy gradient.
\newblock {\em arXiv preprint arXiv:1704.02399}, 2017.

\bibitem{lu2019scaling}
J.~Lu, Y.~Lu, and J.~Nolen.
\newblock Scaling limit of the stein variational gradient descent: The mean
  field regime.
\newblock {\em SIAM Journal on Mathematical Analysis}, 51(2):648--671, 2019.

\bibitem{menozzi2021density}
S.~Menozzi, A.~Pesce, and X.~Zhang.
\newblock Density and gradient estimates for non degenerate brownian sdes with
  unbounded measurable drift.
\newblock {\em Journal of Differential Equations}, 272:330--369, 2021.

\bibitem{messaoud2024s}
S.~Messaoud, B.~Mokeddem, Z.~Xue, L.~Pang, B.~An, H.~Chen, and S.~Chawla.
\newblock S$2$ac: Energy-based reinforcement learning with stein soft actor
  critic.
\newblock {\em arXiv preprint arXiv:2405.00987}, 2024.

\bibitem{nusken2021stein}
N.~N{\"u}sken and DR~Renger.
\newblock Stein variational gradient descent: many-particle and long-time
  asymptotics.
\newblock {\em arXiv preprint arXiv:2102.12956}, 2021.

\bibitem{otto2000generalization}
F.~Otto and C.~Villani.
\newblock Generalization of an inequality by talagrand and links with the
  logarithmic sobolev inequality.
\newblock {\em Journal of Functional Analysis}, 173(2):361--400, 2000.

\bibitem{pu2017vae}
Y.~Pu, Z.~Gan, R.~Henao, C.~Li, S.~Han, and L.~Carin.
\newblock {VAE} learning via {S}tein variational gradient descent.
\newblock In {\em Advances in Neural Information Processing Systems (NIPS)},
  pages 4236--4245, 2017.

\bibitem{salim2022convergence}
A.~Salim, L.~Sun, and P.~Richtarik.
\newblock A convergence theory for svgd in the population limit under
  talagrand’s inequality t1.
\newblock In {\em International Conference on Machine Learning}, pages
  19139--19152. PMLR, 2022.

\bibitem{shi2024finite}
J.~Shi and L.~Mackey.
\newblock A finite-particle convergence rate for stein variational gradient
  descent.
\newblock {\em Advances in Neural Information Processing Systems}, 36, 2024.

\bibitem{song2020score}
Y.~Song, J.~Sohl-Dickstein, D.~P Kingma, A.~Kumar, S.~Ermon, and B.~Poole.
\newblock Score-based generative modeling through stochastic differential
  equations.
\newblock {\em arXiv preprint arXiv:2011.13456}, 2020.

\bibitem{sun2023convergence}
L.~Sun, A.~Karagulyan, and P.~Richtarik.
\newblock Convergence of stein variational gradient descent under a weaker
  smoothness condition.
\newblock In {\em International Conference on Artificial Intelligence and
  Statistics}, pages 3693--3717. PMLR, 2023.

\bibitem{tao2019variational}
C.~Tao, S.~Dai, L.~Chen, K.~Bai, J.~Chen, C.~Liu, R.~Zhang, G.~Bobashev, and
  L.~C. Duke.
\newblock Variational annealing of gans: A langevin perspective.
\newblock In {\em International conference on machine learning}, pages
  6176--6185. PMLR, 2019.

\bibitem{vempala2019rapid}
S.~Vempala and A.~Wibisono.
\newblock Rapid convergence of the unadjusted langevin algorithm: Isoperimetry
  suffices.
\newblock {\em Advances in neural information processing systems}, 32, 2019.

\bibitem{villani2008optimal}
C.~Villani.
\newblock {\em Optimal transport}, volume 338 of {\em Grundlehren der
  mathematischen Wissenschaften [Fundamental Principles of Mathematical
  Sciences]}.
\newblock Springer-Verlag, Berlin, 2009.
\newblock Old and new.

\bibitem{zhang2018learning}
R.~Zhang, C.~Li, C.~Chen, and C.~Carin.
\newblock Learning structural weight uncertainty for sequential
  decision-making.
\newblock In {\em International Conference on Artificial Intelligence and
  Statistics}, pages 1137--1146. PMLR, 2018.

\bibitem{zhang2019scalable}
R.~Zhang, Z.~Wen, C.~Chen, and L.~Carin.
\newblock Scalable thompson sampling via optimal transport.
\newblock {\em arXiv preprint arXiv:1902.07239}, 2019.

\bibitem{zhuo2018message}
J.~Zhuo, C.~Liu, J.~Shi, J.~Zhu, N.~Chen, and B.~Zhang.
\newblock Message passing stein variational gradient descent.
\newblock In {\em International Conference on Machine Learning}, pages
  6018--6027. PMLR, 2018.

\end{thebibliography}

\newpage
\appendix

\part*{Appendix}

\tableofcontents

\clearpage

\newpage
\section{Fig.~\ref{fig:images} in larger format}

\begin{figure}[ht!]
    \centering
    \begin{subfigure}[b]{0.9\textwidth} 
        \centering
        \includegraphics[width=\textwidth]{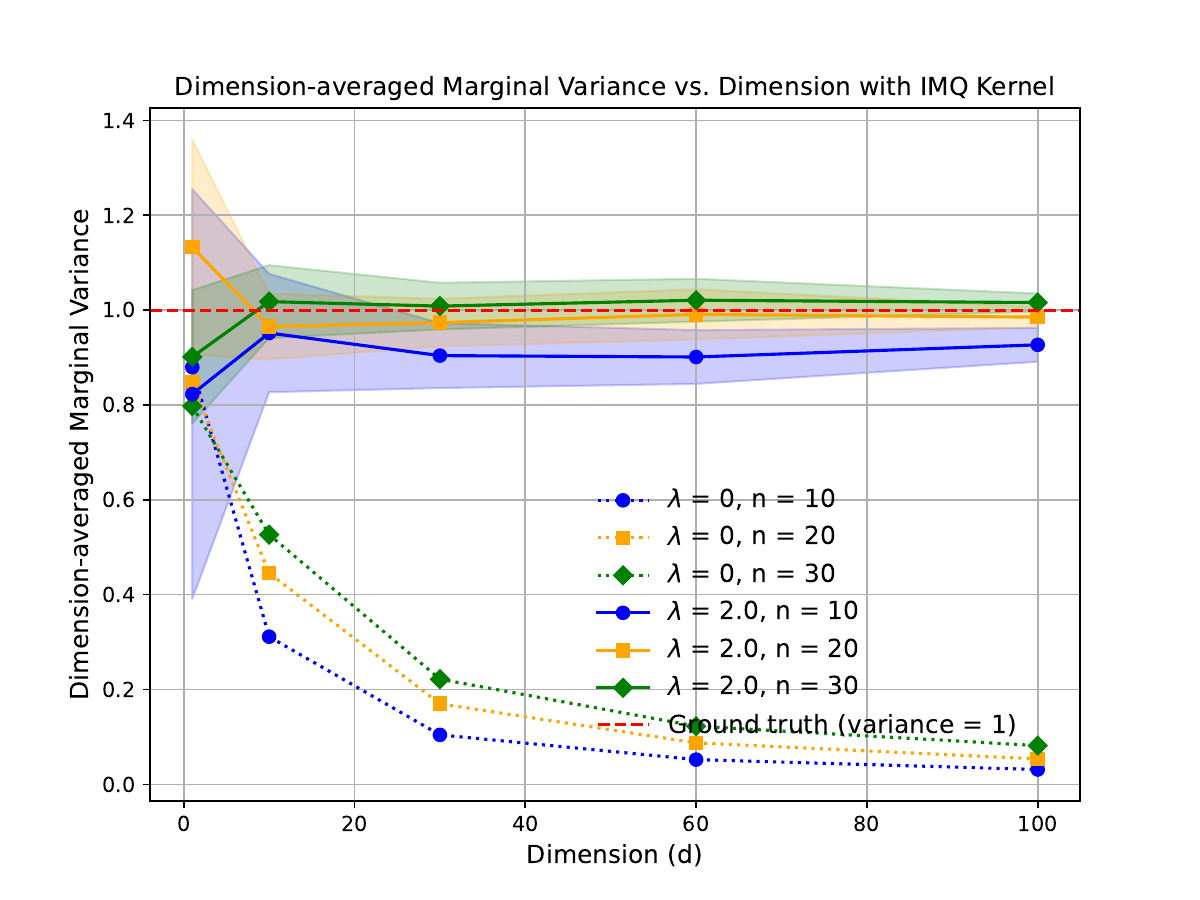} 
        \caption{IMQ kernel}
        \label{fig:image10}
    \end{subfigure}
    \hfill 
    \begin{subfigure}[b]{1\textwidth} 
        \centering
        \includegraphics[width=0.9\textwidth]{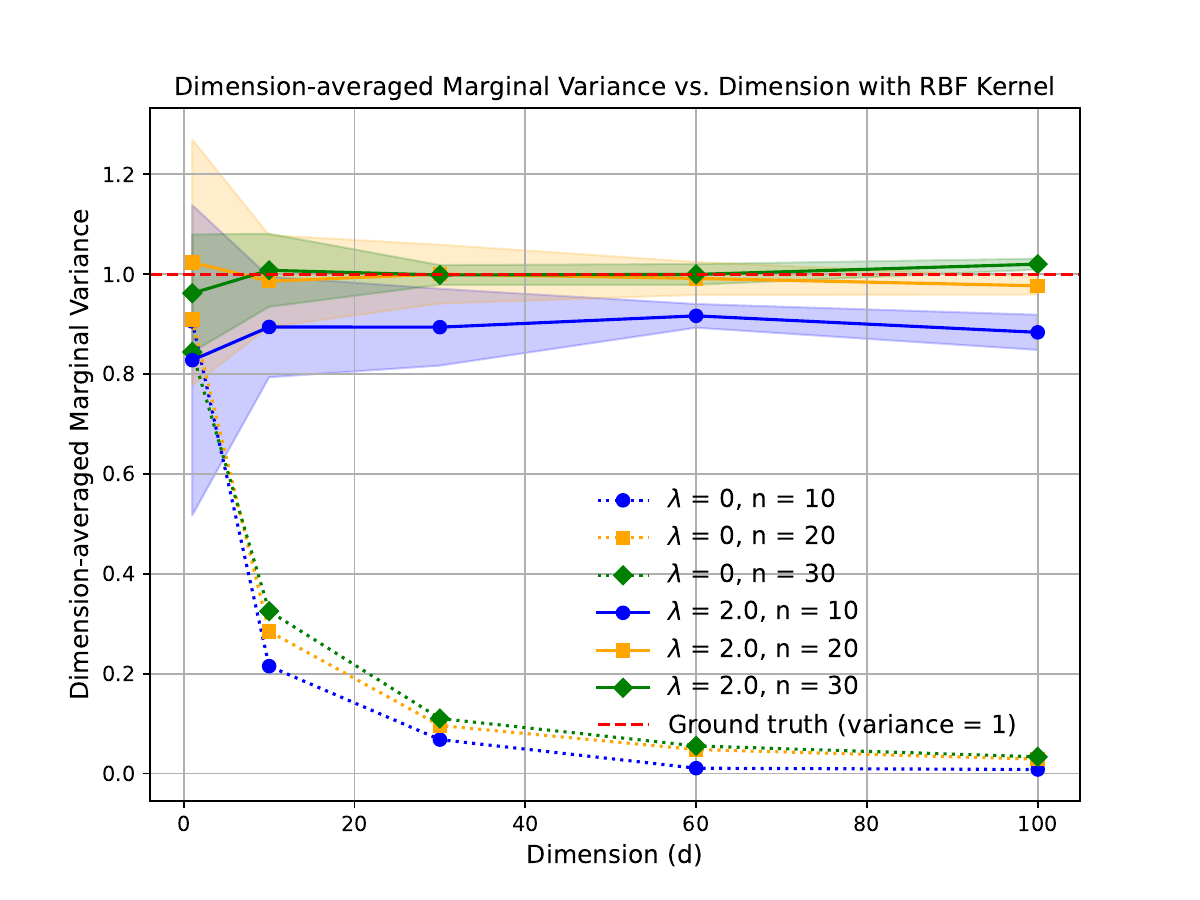} 
        \caption{RBF kernel}
        \label{fig:image20}
    \end{subfigure}
    \caption{Dimension-averaged Marginal Variance of SVGD and noisy SVGD at convergence for sampling from a standard Gaussian.}
    \label{fig:images0}
\end{figure}

\newpage

\section{Notations}

 We denote by $[n]$ the set of integers $\{1,\dots, n\}$.

We denote by $\langle \cdot, \cdot \rangle$ and $\| \cdot \|$
the inner product and the corresponding norm in a Euclidean
space. We use the same notation in an infinite dimensional space.

Let $d\in\bN^*$. For $k\in\bN\cup\{\infty\}$, we denote by $C^k(\mathbb{R}^d,\mathbb{R}^q)$ the
set of functions which are continuously differentiable up to the order $k$.  We
denote by $C_c(\bR^d,\bR)$ the set of $\bR^d\to\bR$ continuous functions with
compact support. Given $p\in \bN^*\cup \{\infty\}$, we denote as
$C_c^p(\bR^d,\bR)$ the set of compactly supported $\bR^d\to\bR$ functions which
are continuously differentiable up to the order $p$. 



The notation $f_\# \mu$ stands for the pushforward of the measure
$\mu$ by the map $f$, that is, $f_\# \mu= \mu\circ f^{-1}$.

For $t\geq 0$, we define the projections $\pi_t$ and $\pi_{[0,t]}$ as 
$\pi_t : (\bR^d)^{[0,\infty)} \to \bR^d, x \mapsto x_t$, and $\pi_{[0,t]} : (\bR^d)^{[0,\infty)} \to (\bR^d)^[0,t], x \mapsto(x_u\, : \, u\in[0,t])$

Define: 
$$
\cP_2(\cC) = \{\rho\in \cP(\cC)\,:\, \forall T>0,\, \int \sup_{t\in[0,T]} \|x_t\|^2 \dr \rho(x)<\infty\}\,.
$$
For every $\rho,\rho'\in\cP_2(\cC)$, we define:
$$
{\mathsf W}_2(\rho,\rho') = \sum_{n=1}^\infty 2^{-n} (1\wedge { W}_2( (\pi_{[0,n]})_\#\rho,(\pi_{[0,n]})_\#\rho'))\,,
$$
where we equipped the space of the $[0,n]\to \bR^d$ continuous function with the uniform norm for every $n\in \bN^*$. 
We equip $\cP_2(\cC)$ with the distance ${\mathsf W}_2$. By \cite[Prop. 1]{bianchi2024long}, $\cP_2(\cC)$ is a Polish space.

For $\rho\in\mathcal{P}_2(\mathcal{C})$, we denote $$\rho_t:=(\pi_t)_\#\rho\,.$$


\section{Proof of Lem.~\ref{lem:stab}}

In this section, we let Assumptions~\ref{hyp:algo} and~\ref{hyp:stab} hold. Additionally, we assume $\lambda>0$. Furthermore, \( C > 0 \) will denote a generic and sufficiently large constant independent of \( k \) and \( n \).

We define:
\begin{equation*}\label{eq:I}
    I_{k,n} := \frac 1n \sum_{i\in [n]} F(X^{i,n}_k) \,. 
\end{equation*}
We will proceeds in three steps. First, we will obtain:
\begin{lemma}\label{lem:I1}The following holds:
    $$\sup_{k,n}\bE(I_{k,n}) <\infty \,.$$
\end{lemma}
Secondly:
\begin{lemma}\label{lem:I2}The following holds:
    $$\sup_{k,n}\bE(I_{k,n}^2) <\infty \,.$$
\end{lemma}
The latter lemma gives a bound on the cross terms of the form \(\mathbb{E}(F(X^{i,n}_k)F(X^{j,n}_k))\) for \(i \neq j\). With this at hand, we obtain:
\begin{lemma}\label{lem:I3}The following holds:
  $$\sup_{k,n}\bE(F(X_k^{1,n})^2) <\infty \,.$$
\end{lemma}
Since, $F(x) \ge c' \norm{x^2} - C$ by Assumtion~\ref{hyp:stab}. By Lem.~\ref{lem:I3}, Lem.~\ref{lem:stab} is proven.

\newpage
\paragraph*{Proof of Lem.~\ref{lem:I1}}
  By Taylor-Lagrange formula, there exists $t_{k+1}^{i,n}\in [0,1]$ such that:
  \begin{multline}\label{eq:Taylor}
    F(X^{i,n}_{k+1}) = F(X^{i,n}_k) + \ps{\nabla F(X_k^{i,n}), X^{i,n}_{k+1}-X^{i,n}_{k}} +  \\
    \frac 1{2}\p{  \p{X^{i,n}_{k+1}-X^{i,n}_{k} }^T H_F\p{X^{i,n}_{k+1} + t^{i,n}_{k+1}\p{X^{i,n}_{k+1}-X^{i,n}_{k}}}(X^{i,n}_{k+1}-X^{i,n}_{k})}\,.
  \end{multline}
We recall the iteration Eq.~\eqref{eq:algo}
\begin{multline*}
  X_{k+1}^{i,n} - X_k^{i,n} = -{\frac {\gamma_{k+1}}n \sum_{j\in[n]}\p{K(X_k^{i,n},X_k^{j,n})\nabla F(X_k^{j,n}) -\nabla_2 K(X^{i,n}_k, X^{j,n}_k)}  } \\ -  \lambda\gamma_{k+1}\nabla F(X_k^{i,n})  + \sqrt{2\gamma_{k+1}\lambda }\xi^{i,n}_{k+1}\,.
\end{multline*}
By Assumption~\ref{hyp:stab}, $\norm{H_F(x)}_{op}\le C$ for every $x\in\bR^d$. Using Eq.~\eqref{eq:Taylor}, we obtain
\begin{multline*}
  {F(X^{i,n}_{k+1})} \le  {F(X^{i,n}_k)}  -\frac{\gamma_{k+1}}{n}\sum_{j\in[n]}\ps{\nabla F(X^{i,n}_k),\nabla F(X^{j,n}_k)} K(X_k^{i,n},X_k^{j,n}) \\
  + \frac{\gamma_{k+1}}n\sum_{j\in [n]} \ps{\nabla F(X^{i,n}_k),\nabla_2 K(X_k^{i,n},X_k^{j,n})} + \sqrt{2\gamma_{k+1}\lambda}\ps{\nabla F(X_k^{i,n}), \xi^{i,n}_{k+1}} \\
  +{C\gamma_{k+1}^2}\p{\norm{\frac 1n\sum_{j\in [n]} K(X^{i,n}_k, X^{j,n}_k)\nabla F(X^{j,n}_k)}^2+\norm{\frac 1n \sum_{j\in [n]}\nabla_2 K(X^{i,n}_k,X^{j,n}_k)}^2 } \\
  - \lambda\gamma_{k+1}\norm{\nabla F(X^{i,n}_k)}^2+ {C\lambda^2\gamma_{k+1}^2} \norm{\nabla F(X^{i,n}_k)}^2 + C\lambda\gamma_{k+1} \norm{\xi^{i,n}_{k+1}}^2   \,.
\end{multline*}
Note that
\[
  \begin{split}
  \frac 1n \sum_{j\in[n]} \ps{\nabla F(X^{i,n}_k), \nabla_2 K(X^{i,n}_k, X^{j,n}_k)}\le C \norm{\nabla F(X^{i,n}_k)}\,.
  \end{split}
  \]
We remark that for an arbitrary $\Phi = (\Phi_\ell)_{\ell\in[d]}\in \cH$, and for every $y\in\bR^d$
\[
  \norm{\Phi(y)}^2 = \sum_{\ell\in [d]} \ps{\Phi_\ell, K(\cdot,y) }^2_{\cH_0} \le \sum_{\ell\in[d]} \norm{\Phi_\ell}_{\cH_0}^2\norm{K(\cdot,y)}_{\cH_0}^2 \le C\norm{\Phi}_\cH^2\,.
\]
Therefore,
\[
\norm{\nabla_2 K(X^{i,n}_k,X^{j,n}_k)}^2\le C \norm{ \nabla_2 K(\cdot,X^{j,n}_k)}_\cH^2 \le C \,,
\]
and
\[
\norm{\sum_{j\in [n]} K(X^{i,n}_k, X^{j,n}_k)\nabla F(X^{j,n}_k)}^2 \le C \norm{\sum_{j\in [n]} K(\cdot, X^{j,n}_k)\nabla F(X^{j,n}_k)}^2_\cH\,.
\]
Consequently, we obtain
\begin{multline}\label{eq:F}
  {F(X^{i,n}_{k+1})} \le  {F(X^{i,n}_k)}  -\frac{\gamma_{k+1}}{n}\sum_{j\in[n]}\ps{\nabla F(X^{i,n}_k),\nabla F(X^{j,n}_k)} K(X_k^{i,n},X_k^{j,n}) \\
  + {\gamma_{k+1}}C\norm{ \nabla F(X^{i,n}_k)} +\sqrt{2\gamma_{k+1}\lambda}\ps{\nabla F(X_k^{i,n}), \xi^{i,n}_{k+1}} \\
  +{C\gamma_{k+1}^2}\p{\norm{\frac 1n \sum_{j\in [n]} K(\cdot, X^{j,n}_k)\nabla F(X^{j,n}_k)}_\cH^2+1} \\
  - \lambda\gamma_{k+1} (1- C\lambda\gamma_{k+1})\norm{\nabla F(X^{i,n}_k)}^2 + {C\lambda \gamma_{k+1}} \norm{\xi^{i,n}_{k+1}}^2  \,.
\end{multline}
We define $J_{k,n} := \frac 1n \sum_{i\in[n]} \norm{\nabla F(X^{i,n}_k)}^2$.
 Hence, we obtain
\begin{multline*}
  {I_{k+1,n}} \le  I_{k,n}  -{\gamma_{k+1}}(1-C\gamma_{k+1})\norm{\frac 1{n}\sum_{j\in [n]} K(\cdot, X^{j,n}_k)\nabla F(X^{j,n}_k)}_\cH^2 \\
   - \lambda \gamma_{k+1}(1 - C\lambda\gamma_{k+1})J_{k,n}+{\gamma_{k+1}}C \sqrt{J_{k,n}} \\
  +\sqrt{2\gamma_{k+1}\lambda}\frac 1n \sum_{i\in[n]}\ps{\nabla F(X_k^{i,n}), \xi^{i,n}_{k+1}} + {C \lambda \gamma_{k+1}} \frac 1n\sum_{i\in[n]}\norm{\xi^{i,n}_{k+1}}^2 
  +{C\gamma_{k+1}^2}\,.
\end{multline*}
By Assumption~\ref{hyp:stab}, $c' I_{k,n} - C\le J_{k,n}\le C'  I_{k,n} + C $. Hence, for $k$ large enough, there exist a constant $c>0$ small enough
  \begin{multline}\label{eq:Ib}
    {I_{k+1,n}} \le  I_{k,n}(1- c\gamma_{k+1})   
  +C {\gamma_{k+1}}\sqrt{C'I_{k,n} +C} \\
    +\sqrt{2\gamma_{k+1}\lambda}\frac 1n\sum_{i\in[n]}\ps{\nabla F(X_k^{i,n}), \xi^{i,n}_{k+1}} + {C\lambda \gamma_{k+1}} \frac 1n\sum_{i\in[n]}\norm{\xi^{i,n}_{k+1}}^2 
    +C\gamma_{k+1}\,.
  \end{multline}
Taking the expectation in Eq.~\eqref{eq:Ib}, we obtain by Assumption~\ref{hyp:algo}:
\[
\esp{I_{k+1,n}}  \le \esp{I_{k,n}}(1-c \gamma_{k+1}) + C\gamma_{k+1}\sqrt{C'\esp{I_{k,n}}+C} + C\gamma_{k+1}\,.
\]
There exists a constant $\kappa$ large enough satisfying
\[
  c\kappa  \ge  C\sqrt{C'\kappa + C} +C\,.
\]
Hence, as soon as there exists $k$ large enough such that $\esp{I_{k,n}} \ge \kappa $, we obtain $\esp{I_{k+1,n}}\le \esp{I_{k,n}}$. Consequently, since $\kappa$ is independent of $n$, Lem.~\ref{lem:I1} is proven.

\paragraph*{Proof of Lem.~\ref{lem:I2}}

Raising Eq.~\eqref{eq:Ib} to the square and taking the expectation, we obtain for $k$ large enough, the existence of a constant $\tilde c> 0$ small enough, such that
\begin{equation*}
  \esp{I_{k+1,n}^2} \le  \esp{I_{k,n}^2}(1- \tilde c\gamma_{k+1}) + C\gamma_{k+1}\esp{I_{k,n}^2}^{3/4} +C\gamma_{k+1}\esp{I_{k,n}^2}^{1/2} + C\gamma_{k+1}^2\,.\\
\end{equation*}
As in the proof of Lem.~\ref{lem:I1}, Lem.~\ref{lem:I2} is proven.

\paragraph*{Proof of Lem.~\ref{lem:I3}}

By Assumption~\ref{hyp:algo}, the sequence $(X_k^{i,n})_{i\in[n]}$ is exchangeable, i.e. the sequence is invariant in law by permutation of the indices $i\in[n]$. 
Then, by Lem.~\ref{lem:I2}, we obtain
\begin{equation}\label{eq:boundF}
   \sup_{k,n} \p{\frac {n-1}{n} \esp{F{(X^{1,n}_k)}F{(X^{2,n}_k)}} + \frac 1n \esp{F({X_k^{1,n}})^2} }<\infty \,.
\end{equation}
Going back to Eq.~\eqref{eq:F} and raising it to the square and taking the expectation, using $\norm{\nabla F(x)}^2\le C (\abs{F(x)}+1) $ and the exchangeability of $(X^{k,n}_i)_{i\in[n]}$, we obtain the existence of a constant $\tilde c$ small enough, such that
\begin{multline}\label{eq:filalF2}
  \esp{{F(X^{1,n}_{k+1})}^2} \le  \esp{F(X^{1,n}_k)^2}(1-\tilde c \gamma_{k+1}) \\+C{\gamma_{k+1}}\p{\frac{n-1}n\bE\abs{\ps{\nabla F(X^{1,n}_k),\nabla F(X^{2,n}_k)} {F(X^{1,n}_k)}}+ \frac {1}n\esp{\norm{\nabla F(X^{1,n}_k)}^2 \abs{F(X^{1,n}_k)}}}  \\
 + C{\gamma_{k+1}}\esp{\norm{\nabla F(X^{1,n}_k)}\abs{F(X^{1,n}_k)}} + C\gamma_{k+1}\esp{\abs{F(X_k^{i,n})}} + C\gamma_{k+1}\,.\\
\end{multline}
In the above inequality, we didn't write the terms in \(\gamma_k^2\) as they are dominated by the terms in \(\gamma_k\).
In the rest of the proof, we bound the second term on the right-hand side of the above inequality. The other terms are easier and are left to the reader.
By Cauchy-Schwarz inequality, we obtain
\[
  \begin{split}
    \esp{\ps{\nabla F(X^{1,n}_k),\nabla F(X^{2,n}_k)} {F(X^{1,n}_k)}} &\le \sqrt{\esp{F(X^{1,n}_k)^2}}\sqrt{\esp{\norm{\nabla F(X_k^{1,n})}^2\norm{\nabla F(X_k^{2,n})}^2}}   \,.\end{split}
\]
Moreover, by Assumption~\ref{hyp:stab}, $\|\nabla F(x)\|^2 \le C'F(x) + C$, and
\[
  \norm{\nabla F(X_k^{1,n})}^2\norm{\nabla F(X_k^{2,n})}^2 \le C'^2 F(X^{1,n}_k)F(X^{2,n}_k) +  CC'F(X^{1,n}_k) + C'CF(X^{2,n}_k) + C^2\,.  
\]
By Eq.~\eqref{eq:boundF},
\[
  \bE\abs{\norm{\nabla F(X_k^{1,n})}^2\norm{\nabla F(X_k^{2,n})}^2} \le C(1 + \sqrt{\esp{F(X_k^{1,n})^2}})\,.
\]
Hence, we obtain
\[
  \bE\abs{\ps{\nabla F(X^{1,n}_k),\nabla F(X^{2,n}_k)} {F(X^{1,n}_k)}} \le   C\p{{\esp{F(X^{1,n}_k)^2}}^{1/2}+ {\esp{F(X^{1,n}_k)^2}}^{3/4}}\,.
\]
By Eq.~\eqref{eq:boundF}, we also obtain
\[
  {\frac {1}n\esp{\norm{\nabla F(X^{1,n}_k)}^2 \abs{F(X^{1,n}_k)}}} \le \frac Cn (\esp{F(X^{1,n}_k)^2 } + \bE{\abs{F(X^{1,n}_k)}}) \le C \,.
\]
Going back to Eq.~\eqref{eq:filalF2}, we obtain
\begin{equation*}
  \esp{{F(X^{1,n}_{k+1})}^2} \le  \esp{F(X^{1,n}_k)^2}(1-\tilde c \gamma_{k+1}) +    C\gamma_{k+1}( {\esp{F(X^{1,n}_k)^2}}^{\tfrac 12}+ {\esp{F(X^{1,n}_k)^2}}^{\tfrac 34}+ 1)\,.
\end{equation*}
Hence, $\sup_{k,n} \esp{F(X_k^{1,n})^2} <\infty$.
\section{Tightness results}
 We define the \textit{intensity} of a random variable $\nu: \Omega \to \cP_2(\bR^d)$,
as the measure $\bI(\nu)\in \cP(\bR^d)$ that satisfies 
\[
\forall A \in \cB(\bR^d), \quad \bI(\nu)(A) := \bE\p{\nu(A)}. 
\]
\begin{lemma}
\label{lem:meleardGen}
A sequence $(\mu_n)$ of random variables on $\mathcal{P}_2(\bR^d)$
is tight if  the sequence $(\bI (\mu_n))$ is 
relatively compact in $\cP_2(\bR^d)$. 
\end{lemma}
\begin{proof}
This proof is identical to the one presented in \cite[Lem.~2]{bianchi2024long}.
\end{proof}
\subsection{Proof of Th.~\ref{th:tight} and Prop.~\ref{prop:tightm}}

First, we state a more general result, which is a consequence of Lem.~\ref{lem:stab}.
\begin{lemma}{\cite[Prop.~4]{bianchi2024long}}\label{lem:tightG}
       The collection of measure $(\bI(m_t^n))_{t,n}$ is relatively compact in $\cP_2(\cC)$.
       Moreover, the collection of random variables $(m_t^n)_{t,n}$ is tight.
\end{lemma}

Next, as the consequence of the above lemma, we obtain the proof of Prop.~\ref{prop:tightm}.
\paragraph{Proof of Prop.~\ref{prop:tightm}}
This is given by~\cite[Lem.~8]{bianchi2024long}.

\paragraph{Proof of Th.~\ref{th:tight}} Remark that $(\pi_0)_\# m_{\tau_k}^n =\mu_k^n$, for every $k$. Hence, $(\pi_0)_\# \bI(m_{\tau_k}^n)$.
For a compact set $\cK\subset \cP_2(\cC)$, one can obtain that $(\pi_0)_\# \cK$ is a compact set in $\cP_2(\bR^d)$. Consequently, since $\bI(m_t^n)_{t,n}$ is relatively compact in $\cP_2(\cC)$ by Lem.~\ref{lem:tightG}, $(\bI(\mu_k^n))_{k,n}$ is relatively compact in $\cP_2(\bR^d)$. This yields the first claim of the theorem, by Lem.~\ref{lem:meleardGen}.

Moreover,
\[
    \bI(\bar \mu_k^n) =\frac{\sum_{i\in[k]} \gamma_i \bI(\mu^n_i)  }{\sum_{i\in [k]}\gamma_i } \,.
\]
Since,  $(\bI(\mu_k^n))_{k,n}$ is relatively compact in $\cP_2(\bR^d)$,  the same holds for $(\bI(\bar \mu_k^n))_{k,n}$. The proof is left to the reader.
By Lem.~\ref{lem:meleardGen}, this finishes the proof.

\section{The McKean-Vlasov measures}

For every $\mu\in \cP_2(\bR^d)$, we define $L(\mu)$ which, to every test
function $\phi\in C_c^2(\bR^d,\bR)$, associates the function $L(\mu)(\phi)$
given by 
\begin{equation}\label{eq:L} 
L(\mu)(\phi)(x)=  \ps{\int ( -K(x,y)\nabla F(y) +\nabla_2 K(x,y))\dr \mu(y) -\lambda\nabla F(x),\nabla \phi(x) } + \lambda \Delta \phi(x)\,.
\end{equation}
Let $(X_t:t\in [0,\infty))$ be the canonical process on $\cC$.
Denote by $(\cF_t^X)_{t\geq 0}$ the natural filtration (\emph{i.e.}, the filtration generated by $\{X_s:0\leq s\leq t\}$).

By a weak solution of the McKean-Vlasov SDE in Definition~\ref{def:V}, we mean a solution of the martingale problem defined hereafter. Hence, for the rest of the appendix, we will take the subsequent definition of \(\sV_2\) into account.
\begin{definition}
\label{def:Vbis}
We say that a measure $\rho\in \cP_2(\cC)$ belongs to the class $\sV_2$ if,  for every $\phi\in C_c^2(\bR^d,\bR)$,
$$
\phi(X_t) -\int_0^t L(\rho_s)(\phi)(X_s)ds
$$
is a $(\cF_t^X)_{t\geq 0}$-martingale on the probability space $(\cC,\cB(\cC),\rho)$.
\end{definition}

We define the function
\[
    b(x,y) := - K(x,y)\nabla F(y) + \nabla_2 K(x,y) -\lambda \nabla F(x)
\]
With a slight abuse of notation, for a measure $\mu\in\cP(\bR^d)$, we denote
$b(x,\mu):=\int b(x,y)\dr \mu(y)$.
Therefore,  $L(\mu)(\phi)(x) = \ps{b(x,\mu),\nabla\phi(x)} + \lambda \Delta \phi(x) $. When $b$ is continuous with linear growth, i.e. $\norm{b(x,y)} \le C(1+\norm{x} +\norm{y})$ for every $x,y\in\bR^d$, the space $\sV_2$ is Polish.
\begin{lemma}{\cite[Prop. 3]{bianchi2024long}}\label{lem:V2}
    Let Assumption~\ref{hyp:stab} holds. $\sV_2$ is closed. Consequently, the space $(\sV_2,\sW_2)$ is Polish.
\end{lemma}

In the rest of the appendix, we will use the following property when we want to obtain properties on the space $\sV_2$
\begin{prop}\label{prop:ito}Let Assumption~\ref{hyp:stab} holds. 
  Let $\psi\in C_c^{\infty}(\bR_+\times\bR^d)$, then for every $t_2\ge t_1 \ge 0$, we obtain
  \begin{multline}\label{eq:ito}
  \int \psi(t_2,x) \dr \rho_{t_2}(x)-\int \psi(t_1,x) \dr \rho_{t_1}(x) = \int_{t_1}^{t_2}\!\!\!\int \6_t\psi(t,x) \dr\rho_t(x)\dr t\\
   + \int_{t_1}^{t_2} \!\!\!\int \ps{\nabla \psi(t,x), b(x,\rho_t)} \dr \rho_t(x)\dr t + \lambda \int_{t_1}^{t_2}\!\!\! \int \Delta\psi(t,x)\dr \rho_t(x)\dr t\,.
  \end{multline}
\end{prop}
\begin{proof}
    Let $\phi\in C_c^\infty(\bR^d)$.
    Let $\rho\in\sV_2$. By Def.~\ref{def:Vbis}, the function
    \[
    t\in\bR_+\mapsto \int \phi(x)\dr \rho_t(x) - \int_0^t \int L(\rho_s)(\phi)(x)\dr\rho_s(x)\dr s\,
    \]
    is constant. 
    Hence, the function $\Phi(t) := \int\phi(x)\dr \rho_t(x) $ is absolutely continuous, with derivative $\Phi'(t) = \int L(\rho_t)(\phi)(x)\dr\rho_t(x)$, which is bounded on compacts under Assumption~\ref{hyp:stab}.
    Let $\eta\in C^\infty_c(\bR_+)$, by an integration by parts,
    we obtain for every $t_2>t_1\ge 0$
    \[
        \int_{t_1}^{t_2} \Phi(t)\eta(t) \dr t= \int_{t_1}^{t_2} \Phi'(t)\eta(t) + \Phi(t)\eta'(t) \dr t\,.
    \]
   Hence, if we define $\psi(t,x) := \psi(x)\eta(t)$, we obtain Eq.~\eqref{eq:ito}. It suffices to remark that functions of the form $(t,x)\mapsto \psi(x)\eta(t)$ for every $(\eta,\phi)\in C_c^\infty(\bR_+)\times C_c^\infty(\bR^d)$ are dense in $C_c^\infty(\bR_+\times\bR^d)$, and the proof is finished.
\end{proof}

\begin{lemma}
  \label{lem:density}
   Let Assumptions~\ref{hyp:stab} and~\ref{hyp:holder} hold. Moreover, we assume $\lambda >0$.
   Let $\rho\in\sV_2$. For every $t>0$, $\rho_t$ admits a density $x\mapsto \varrho(t,x) \in C^1(\bR^d,\bR)$. 
   Moreover, for every $R>0, t_2>t_1>0$, there exists a constant $C_{R,t_1,t_2}>0$ such that:
    \begin{equation}\label{eq:binfdensity}
    \inf_{t\in[t_1,t_2], \norm{x}\le R}  \varrho(t,x) \ge C_{R,t_1,t_2}  \,,
    \end{equation}
      and there exist a constant $C_{t_1,t_2}>0$, such that
  \begin{equation}\label{eq:boundrho}
      \sup_{x\in\bR^d, t\in[t_1,t_2]} \norm{\nabla \varrho(t,x)}  + \varrho(t,x)\le C_{t_1,t_2}\,.
  \end{equation}
  Additionally,
    \begin{equation}
      \label{eq:bsupdensity}
    \sup_{t\in[t_1,t_2]}\int (1+\|x\|^2)\norm{\nabla \varrho(t,x)}\dr x <\infty \,.
    \end{equation}
Finally, 
\begin{equation}\label{eq:KLbound}
    \sup_{\rho\in\cK} \KL{\rho_{t_1}}{\pi} <\infty\, ,
\end{equation}
for every compact set $\cK\subset \sV_2$.
\end{lemma}
    \begin{proof}
  The result is an application of~\cite[Th. 1.2]{menozzi2021density} with the non homogeneous vector field
  $\tilde b(t,x) :=\int  b(x,y)\dr \rho_t(y)$. The proof consists in verifying the conditions of the latter theorem.
By Assumptions~\ref{hyp:stab} and~\ref{hyp:holder}, for every $(x,y,T)\in(\bR^d)^2\times  \bR_+$,
\begin{align*}
\sup_{t\in[0,T]}\norm{\tilde b(t,x) -\tilde b(t,y)} &\le \lambda \norm{\nabla F(x) -\nabla F(y)} \\ & + \sup_{t\in[0,T]}\int \norm{\nabla_2 K(x,z) -\nabla_2 K(y,z)}\dr\rho_t(z)\\&+\sup_{t\in[0,T]}\int \norm{\nabla F(z)}\abs{K(x,z) - K(y,z)}\dr\rho_t(z)\\ 
                                                    &\le  C (\norm{x-y}^\beta\vee \norm{x-y})\,,
\end{align*}
Moreover,
\begin{equation}
  \label{eq:borne-bt}
   \sup_{t\in[0,T]}\tilde b(t,x) \le C (1+\norm{x} + \int \sup_{t\in[0,T]} \norm{y_t} \dr \rho(y))\le C (1+\norm{x}) \,.
\end{equation}
As $\lambda>0$,  \cite[Th. 1.2]{menozzi2021density} applies: $\rho$ admits a density $x\mapsto \varrho(t,x)\in C^1(\bR^d)$, for $0<t\le T$,
and there exists four constants $(C_{i,T}, \lambda_{i,T})_{i\in[2]}$, such that:
\begin{gather*}
  \frac {1} { C_{1,T}t^{d/2}} \int \exp\p{-\frac {\norm{x-\theta_t(y)}^2}{\lambda_{1,T} t} }\dr\rho_0(y) \le {\varrho(t,x)} \\
 {\varrho(t,x)} \le   \frac {C_{1,T}}{ t^{d/2}}\int \exp\p{-\frac {\lambda_{1,T}}{t} \norm{x-\theta_t(y)}^2}\dr\rho_0(y) \\
 \norm{\nabla \varrho(t,x)} \le \frac {C_{2,T}}{t^{(d+1)/2}}\int \exp\p{-\frac {\lambda_{2,T}}{ t} \norm{x-\theta_t(y)}^2}\dr\rho_0(y)\,,
\end{gather*}
where the map $t\mapsto \theta_t(y)$ is a solution to the ordinary differential equation:
$\dd{\theta_t(y)}{t} = \tilde b(t,\theta_t(y))$ with initial condition $\theta_0(y) = y$. 
By Grönwall's lemma and Eq.~(\ref{eq:borne-bt}), there exists a constant $C_T$ such that $\|\theta_t(y)\| \le C_T\|y\|$, for every $n,y$, and $t\le T$. 
For every $t_1\le t\le t_2$, and every $x$, we obtain using a change of variables:
\begin{align*}
(C_{1,t_2} {t_1}^{d/2})^{-1}\ge {\varrho(t,x)} \ge  C_{1,t_2}{ t_2^{-d/2}}\exp\p{-\frac 2{\lambda_{1,t_2}t_1}\|x\|^2}\int \exp\p{ -\frac {2C_{t_2}}{\lambda_{1,t_2}t_1} \norm{y}^2}\dr\rho_0(y) 
\end{align*}
\begin{multline*}
 \int (1+\|x\|^2)\norm{\nabla \varrho(t,x)} \dr x \\ \le  C_{2,t_2}{t_1^{-(d+1)/2}}\int (1+2\|x\|^2+2C_{t_2}^2\int \|y\|^2d\rho_0(y))\exp\p{- \lambda_{2,t_2} t_2^{-1} \norm{x}^2}\dr x\,, 
\end{multline*}
and $\norm{\nabla \varrho(t,x)} \le C_{2,t_2} t_1^{-(d+1)/2}\,.$
Consequently, $\rho$ satisfies Eq.~\eqref{eq:binfdensity}, Eq.~\eqref{eq:boundrho} and Eq.~\eqref{eq:bsupdensity}.

It remains to obtain Eq.~\eqref{eq:KLbound}. 
Let $\cK\subset \sV_2$ be a compact set and let $\rho\in\cK$.
We observe
\begin{equation}\label{eq:KLetape1}
\KL{\rho_{t_1}}{\pi} \le C + \int \abs{F(x)} \dr \rho_{t_1}(x)  + \int \norm{\log\varrho(t_1,x)} \dr \rho_{t_1}(x)\,.
\end{equation}
By Assumption~\ref{hyp:stab}, since $(\pi_{t_1})_\# \cK$ is a compact set in $\cP_2(\bR^d)$, we obtain
\[
 \sup_{\rho\in\cK}\int \abs{F(x)}\dr\rho_{t_1}(x) \le C \sup_{\rho\in\cK}\int \norm{x}^2\dr \rho_{t_1}(x)\le C \sup_{\mu \in (\pi_{t_1})_\# \cK}  \int \norm{x}^2 \dr \mu(x) < \infty\,.
 \]
 Moreover, by the lower bound and the upper bound on $\varrho$,
 \begin{equation}\label{eq:logrho}
      \norm{\log \varrho(t_1,x)} \le C\p{1 + \norm{x}^2 + \int \norm{y}^2\dr\rho_0(y)}\,.
 \end{equation}
Hence, we obtain
\[
  \sup_{\rho\in\cK} \int \norm{\log\varrho(t_1,x)} \dr\rho_{t_1}(x) <\infty \,.
\]
Finally, applying the latter results in Eq.~\eqref{eq:KLetape1}, we obtain Eq.~\eqref{eq:KLbound}.
\end{proof}

\subsection{Sketch of the proof of Prop~\ref{prop:Lya} using Wasserstein calculus}

We give a sketch of the proof of Lyapunov using Wasserstein calculus~\cite{ambrosio2005gradient}. This proof is not fully rigorous because we would need to check th assumptions of the results from~\cite{ambrosio2005gradient} that we are using. In the next section we give a fully rigorous proof. 

Consider $\rho\in\sV_2$, \textit{i.e.}, the law of a weak solution $(X_t)_t$ of the McKean-Vlasov equation
\[
dX_t = -\int \left( K(X_t,y) \nabla F(y) - \nabla_2 K(X_t,y) \right) d\rho_t(y) \, dt - \lambda \nabla F(X_t) \, dt + \sqrt{2\lambda} \, dW_t.
\] 
For every $t>0$, we denote by $\rho_t$ the marginal of $\rho$. In other words, $\rho_t$ is the law of $X_t$. 

Using integration by parts, the McKean-Vlasov equation can be represented by
\[
dX_t = -  P_\mu\nabla\log \dd{\rho_t}{\pi}(X_t) \, dt - \lambda \nabla F(X_t) \, dt + \sqrt{2\lambda} \, dW_t.
\] 
From this representation, we can derive the continuity equation satisfied by $(\rho_t)_t$:
\[
\frac{\partial \rho_t}{\partial t} = \nabla \cdot (\rho_t \tilde v_t),
\] 
where $\tilde v_t$ is the velocity field
\[
\tilde v_t := -  P_\mu\nabla\log \dd{\rho_t}{\pi} -  \lambda \nabla\log \dd{\rho_t}{\pi}.
\]  
Using the chain rule in the Wasserstein space~\cite[Equation 10.1.16]{ambrosio2005gradient}, we have for every functional $\cF : \cP_2(\bR^d) \to (-\infty, +\infty]$ regular enough that
\[
\frac{d}{dt}\cF(\rho_t) = \ps{\nabla_W \cF(\rho_t), v_t}_{\rho_t},
\]  
where $\ps{\cdot,\cdot}_\rho$ is the standard inner product in $L^2(\rho)$ and $\nabla_W \cF(\rho) \in L^2(\rho)$ is the Wasserstein gradient of $\cF$ at $\rho$. In the case where $\cF(\rho) = \KL{\rho}{\pi}$, we have $\nabla_W \cF(\rho) = \nabla\log \dd{\rho}{\pi}$, therefore 
\begin{align*}
\frac{d}{dt}\cF(\rho_t) &= \left \langle \nabla\log \dd{\rho}{\pi}, -  P_\mu\nabla\log \dd{\rho_t}{\pi} - \lambda \nabla\log \dd{\rho_t}{\pi} \right \rangle_{\rho_t} \\
&= - \left \langle \nabla\log \dd{\rho}{\pi},  P_\mu\nabla\log \dd{\rho_t}{\pi}\right \rangle_{\rho_t} - \lambda \left \langle \nabla\log \dd{\rho}{\pi}, \nabla\log \dd{\rho_t}{\pi}\right \rangle_{\rho_t}. 
\end{align*}
Finally, we use that the kernel integral operator is the adjoint of the injection~\cite{carmeli2010vector} $\iota_\rho : \cH \to L^2(\rho)$. In other words, for every $f \in L^2(\rho), g \in \cH$, $\ps{f,g}_\rho = \ps{P_\rho f,g}_\cH$. Here, this property gives  
\[
\left \langle \nabla\log \dd{\rho}{\pi},  P_\mu\nabla\log \dd{\rho_t}{\pi}\right \rangle_{\rho_t} = \norm{P_\mu\nabla\log \dd{\mu}{\pi}}^2_\cH.
\]
Therefore,
\[
\frac{d}{dt}\cF(\rho_t) = -\norm{P_\mu\nabla\log \dd{\mu}{\pi}}^2_\cH - \lambda\norm{\nabla\log \dd{\mu}{\pi}}^2_{\rho_t}.
\] 
In other words,
\begin{equation*}
          \frac{d}{dt}\KL{\rho_{t}}{\pi}  
             =  -\IS{\rho_t}{\pi} - \lambda \I{\rho_t}{\pi},
    \end{equation*}
and we can conclude by integrating between $t_1>0$ and $t_2>0$.

\subsection{Proof of Prop.~\ref{prop:Lya}}
In this subsection, we let Assumptions~\ref{hyp:stab} and~\ref{hyp:holder} hold. Moreover, we assume $\lambda>0$.

We consider $\rho\in\sV_2$. Moreover, we define two reels $0<t_1 <t_2$.

Let 
\begin{equation}
  v_t(x) := -\int \p{K(x,y)\nabla F(y) - \nabla_2 K(x,y) \dr \rho_t(y)} - \lambda \nabla F(x)  - \lambda \nabla \log \varrho(t,x) \,.
\end{equation}
By Prop~\ref{prop:ito}, with Lem.~\ref{lem:density}, we obtain
  \begin{multline}\label{eq:eqCont}
  \int \psi(t_2,x) \dr \rho_{t_2}(x)-\int \psi(t_1,x) \dr \rho_{t_1}(x)\\ = \int_{t_1}^{t_2}\!\!\!\int \6_t\psi(t,x) \dr\rho_t(x)\dr t
   + \int_{t_1}^{t_2} \!\!\!\int \ps{\nabla \psi(t,x),v_t(x)} \dr \rho_t(x)\dr t \,.
  \end{multline}
Note that the latter quantity is well-defined, since $\int_{t_1}^{t_2}\!\!\int \norm{v_t(x)} \dr\rho_t(x)\dr t $ by Lem.~\ref{lem:density}.
Define a smooth, compactly supported, even function $\eta:\bR^d\to\bR_+$
such that $\int \eta(x)dx=1$, and
define $\eta_\varepsilon(x) := \varepsilon^{-d}\eta(x/\varepsilon)$
for every $\varepsilon>0$. For every $t>0$, we introduce
the density $\varrho_\varepsilon(t,\cdot) := \eta_\varepsilon *
\rho_\varepsilon(t,\cdot)$, and we denote by
$\rho^\varepsilon_t(dx) = \varrho_\varepsilon(t,x)dx$ the corresponding
probability measure.  Finally, we define:
$$
v_t^\varepsilon := \frac{\eta_\varepsilon * (v_t\varrho(t,\cdot))}{\varrho_\varepsilon(t,\cdot)}\,.
$$
With these definitions at hand, it is straightforward to check that Eq.~\eqref{eq:eqCont}
holds when $\rho_t,v_t$ are replaced by $\rho_t^\varepsilon,v_t^\varepsilon$. More specifically,
we shall apply Eq.~(\ref{eq:eqCont}) using a specific smooth function $\psi = \psi_{\varepsilon,\delta,R}$, which we will
define hereafter for fixed values of $\delta,R>0$, yielding our main equation:
\begin{multline}
  \label{eq:continuityBis}
  \int\psi_{\varepsilon,\delta,R}({t_2},x)\varrho_\varepsilon(t_2,x)dx - \int\psi_{\varepsilon,\delta,R}({t_1},x)\varrho_\varepsilon(t_1,x)dx =\\
  \int_{t_1}^{t_2} \int (\partial_t\psi_{\varepsilon,\delta,R}(t,x) +  \ps{\nabla \psi_{\varepsilon,\delta,R}(t,x),v_t^\varepsilon(x)}) \varrho_\varepsilon(t,x)dx dt\,.
\end{multline}
Let $\theta\in C_c^\infty(\bR,\bR)$ be a nonnegative function supported by the interval $[-t_1,t_1]$ and satisfying
$\int \theta(t)dt=1$. For every $\delta\in (0,1)$, define $\theta_\delta(t) = \theta(t/\delta)/\delta$.
We define $\varrho_{\varepsilon,\delta}(\cdot,x) := \theta_\delta* \varrho_{\varepsilon}(\cdot,x)$.
The map $t\mapsto \varrho_{\varepsilon,\delta}(t,\dot)$ is well-defined on $[t_1,t_2]$, non negative, and smooth in both variables $t,x$.
In addition, we define 
$F_\varepsilon := \eta_\varepsilon * F$. Finally, we introduce a smooth function $\chi$ on $\bR^d$ equal to one on the unit ball and to zero outside the ball of radius $2$,
and we define $\chi_R(x) := \chi(x/R)$. For every $(t,x)\in [t_1,t_2]\times \bR$, we define:
\begin{equation}\label{eq:psi}
     \psi_{\varepsilon,\delta,R}(t,x) :=(\log \varrho_{\varepsilon,\delta}(t,x) +  F_\varepsilon(x)) \chi_R(x)\,.
\end{equation}
We extend $\psi_{\varepsilon,\delta,R}$ to a smooth compactly supported function on $\bR_+\times \bR^d$.
We define $U(x,\rho_t) := \int (K(x,y)\nabla F(y) - \nabla_2 K(x,y)\dr\rho_t(y)$.
Applying Eq.~\eqref{eq:continuityBis} with $\psi_{\varepsilon,\delta,R}$,
\[
  \begin{split}
   & \int\psi_{\varepsilon,\delta,R}({t_2},x)d\rho_{t_2}(x) - \int\psi_{\varepsilon,\delta,R}({t_1},x)d\rho_{t_1}(x) \\
  &= \int_{t_1}^{t_2}\!\!\!  \int (\partial_t\psi_{\varepsilon,\delta,R}(t,x) +  \ps{\nabla\psi_{\varepsilon,\delta,R}(t,x),v_t^\varepsilon(x)}) \dr \rho_t^\varepsilon(x)dt \\
  & = \int_{t_1}^{t_2}\!\!\! \int \6_t\varrho_{\varepsilon,\delta} (t,x) \frac{\varrho_{\varepsilon}(t,x)}{\varrho_{\varepsilon,\delta}(t,x)}\chi_R(x)\dr x\dr t \\
  & -\lambda \int_{t_1}^{t_2}\!\!\! \int \ps{\nabla F_\varepsilon(x) + \nabla\log\varrho_{\varepsilon,\delta}(t,x),\frac{ \eta_\epsilon*(\nabla F(\cdot)\varrho(t,\cdot))(x)}{\varrho_\varepsilon(t,x)} + \nabla\log \varrho^{\varepsilon}(t,x) }\chi_R(x)d \rho_t^{\varepsilon}(x)\dr t\\
  & - \int_{t_1}^{t_2}\!\!\! \int \ps{\nabla F_\varepsilon(x) + \nabla\log\varrho_{\varepsilon,\delta}(t,x),\frac{ \eta_\epsilon*(U(\cdot,\rho_t) \varrho(t,\cdot))(x)}{\varrho_\varepsilon(t,x)}  }\chi_R(x)d \rho_t^{\varepsilon}(x)\dr t\\
 & + \int_{t_1}^{t_2}\!\!\!\int (\log \varrho_{\varepsilon,\delta}(t,x) + F_\varepsilon(x))\ps{ \nabla \chi_R(x), v^\varepsilon_t(x) } \dr \rho^\varepsilon_t(x) \dr t \\
 \end{split}
\]
We define, for every $t\in[t_1,t_2]$,
\begin{gather*}
  \Pi_1(t):= \int\psi_{\varepsilon,\delta,R}({t},x)d\rho^\varepsilon_{t}(x), \\
  \Pi_2:= \int_{t_1}^{t_2}\!\!\! \int \6_t\varrho_{\varepsilon,\delta} (t,x) \frac{\varrho_{\varepsilon}(t,x)}{\varrho_{\varepsilon,\delta}(t,x)}\chi_R(x)\dr x\dr t ,\\
  \Pi_3:= \int_{t_1}^{t_2}\!\!\! \int \ps{\nabla F_\varepsilon(x) + \nabla\log\varrho_{\varepsilon,\delta}(t,x),{ \eta_\epsilon*(\nabla F(\cdot)\varrho(t,\cdot))(x)} + \nabla \varrho^{\varepsilon}(t,x) }\chi_R(x)dx\dr t,\\
  \Pi_4:=  \int_{t_1}^{t_2}\!\!\! \int \ps{\nabla F_\varepsilon(x) + \nabla\log\varrho_{\varepsilon,\delta}(t,x),{ \eta_\epsilon*(U(\cdot,\rho_t) \varrho(t,\cdot))(x)} }\chi_R(x)dx\dr t, \\
  \Pi_5:= \int_{t_1}^{t_2}\!\!\!\int (\log \varrho_{\varepsilon,\delta}(t,x) + F_\varepsilon(x))\ps{ \nabla \chi_R(x),v^\varepsilon_t(x) } \varrho^{\varepsilon}(t,x) \dr x \dr t \,.
\end{gather*}
And, it holds:
\begin{equation}\label{eq:Pi}
  \Pi_1(t_2)- \Pi_1(t_1) = \Pi_2 -\lambda \Pi_3 - \Pi_4 + \Pi_5\,.
\end{equation}
We now investigate successively the limit of each term in Eq.~(\ref{eq:Pi}) as $\delta,\varepsilon,R$ successively tend to $0,0,\infty$.

We state a technical result proven at the end of the subsection.
\begin{lemma}\label{lem:continuityRho} For every $\varepsilon, x\in \bR^d$, $t\mapsto \rho^\varepsilon(x,t)$ and $t\mapsto \nabla\varrho^\varepsilon(t,x)$ are absolute continuous functions. Moreover, 
  \[
   \sup_{t\in[t_1,t_2],x\in\bR^d} \abs{\6_t \varrho_{\varepsilon}  (t,x)} \le C_{\varepsilon} \,,
  \]
  for a constant $C_\varepsilon>0$.
\end{lemma}
Since, by Lem.~\ref{lem:density}, the mappings $t\mapsto \varrho_\varepsilon(t,x)$, $x\mapsto  F(x)$ and $x\mapsto \varrho(t,x) $ are continuous, and by Eq~\eqref{eq:binfdensity}, we obtain
\begin{equation}\label{eq:limPsi}
  \lim_{R\to\infty}\lim_{\varepsilon\to 0}\lim_{\delta\to 0} \psi_{\varepsilon,\delta,R}(t,x) = \log{\varrho(t,x)} + F(x)\,.
\end{equation}
By Lem.~\ref{lem:density}, we obtain
\[
    \psi_{\varepsilon,\delta,R}\varrho_\varepsilon(t,x) \le C_R\chi_R(x)\,,
\]
for a constant $C_R$ independent of $\delta,\varepsilon,x$. Hence, we can apply the dominated convergence theorem and we obtain $\lim_{\varepsilon\to 0}\lim_{\delta\to 0}\Pi_1(t) = \int \log (\varrho(t,x) + F(x))\chi_R(x)\dr \rho_t(x)$. Since $\rho_t$ admits moments of order 2, we obtain
\[
  \lim_{R\to\infty}  \lim_{\varepsilon\to 0}\lim_{\delta\to 0}\Pi_1(t) = \KL{\rho_t}{\pi} - \int \exp(-F(x))\dr x\,,
\]
for every $t>0$.

In the following, we will obtain the convergence of $\Pi_2$.
We obtain
\[
\Pi_2 =  \int_{t_1}^{t_2}\!\!\! \int \6_t\varrho_{\varepsilon,\delta} (t,x) \chi_R(x)\dr x\dr t +  \int_{t_1}^{t_2}\!\!\! \int \6_t\varrho_{\varepsilon,\delta} (t,x) \p{\frac{\varrho_{\varepsilon}(t,x)}{\varrho_{\varepsilon,\delta}(t,x)}  -1} \chi_R(x)\dr x\dr t\,.
\]
By Lem.~\ref{lem:continuityRho}, and a convergence dominated argument, we obtain
\[
\lim_{\delta \to 0}\int_{t_1}^{t_2}\!\!\! \int \6_t\varrho_{\varepsilon,\delta} (t,x) \p{\frac{\varrho_{\varepsilon}(t,x)}{\varrho_{\varepsilon,\delta}(t,x)}  -1} \chi_R(x)\dr x\dr t = 0 \,.
\]
Moreover,
\[
 \int_{t_1}^{t_2}\!\!\! \int \6_t\varrho_{\varepsilon,\delta} (t,x) \chi_R(x)\dr x\dr t = \int \varrho_{\varepsilon,\delta}(t_2,x) \chi_R(x)\dr x - \int \varrho_{\varepsilon,\delta}(t_1,x) \chi_R(x)\dr x\,.
\]
Since $\sup_{x\in\bR^d, t>0}{\varrho(t,x)} \le C$, we obtain the by dominated convergence theorem
\[
\lim_{R\to\infty}\lim_{\varepsilon\to 0}\lim_{\delta \to 0} \int \varrho_{\varepsilon,\delta}(t_2,x) \chi_R(x)\dr x - \int \varrho_{\varepsilon,\delta}(t_1,x) \chi_R(x)\dr x = \int \dr \rho_{t_2} - \int \dr \rho_{t_1} = 0\,.
\]
Hence, 
\[
\lim_{R\to\infty}\lim_{\varepsilon\to 0}\lim_{\delta \to 0}  \Pi_2 = 0\,.
\]

Next, we will obtain the convergence of $\Pi_3$. By Lem.~\ref{lem:density} and~\ref{lem:continuityRho}, we obtain 
\[
\lim_{\varepsilon\to 0}\lim_{\delta \to 0} \Pi_3 = \int_{t_1}^{t_2} \int \norm{\nabla F(x) + \nabla \log \varrho(t,x)}^2 \chi_R(x) \rho_t(x)\dr t\,.
\]
And by the monotone convergence theorem, we obtain the limit in $R$:
\[
\lim_{R\to\infty}\lim_{\varepsilon\to 0}\lim_{\delta \to 0} \Pi_3 = \int_{t_1}^{t_2} \int \norm{\nabla F(x) + \nabla \log \varrho(t,x)}^2 \dr \rho_t(x)\dr t\,.
\]

Now, we will obtain the convergence of $\Pi_4$. 
We recall that the kernel $K$ is bounded by Assumption~\ref{hyp:stab}.
First, remark that an integration by parts yields, 
\[
    U(x,\rho_t) = \int K(x,y)\p{\nabla F(y) + \nabla \log \varrho(t,y)             }\dr \rho_t(y)\,,
\]
for every $x\in\bR^d$, which is possible by Lem.~\ref{lem:density}.
 Hence, taking the limit in $\delta,\varepsilon$, we obtain 
\begin{multline*}
    \lim_{\varepsilon\to 0}\lim_{\delta\to 0} \Pi_4 \\ = \int_{t_1}^{t_2}\!\!\!\int\!\!\!\int K(x,y)\ps{\nabla F(x) + \nabla\log\varrho(t,x), \nabla F(y) +\nabla\log\varrho(t,y)} \chi_R(x) \dr\rho_t(x)\dr\rho_t(y)\dr t\,.
\end{multline*}
Since, by Lem.~\ref{lem:density},
$
\sup_{t\in[t_1,t_2]}\int\norm{ \nabla \varrho(t,x)} \dr x <\infty\,,
$
we obtain 
\[
\sup_{t\in[t_1,t_2]} \int  \norm{\nabla F(y) +\nabla \varrho(t,y)} \dr \rho_t(y)<\infty\,.
\]
Hence, taking the limit in $R$,
\begin{multline*}
  \lim_{R\to\infty}  \lim_{\varepsilon\to 0}\lim_{\delta\to 0} \Pi_4 \\ = \int_{t_1}^{t_2}\!\!\!\int\!\!\!\int K(x,y)\ps{\nabla F(x) + \nabla\log\varrho(t,x), \nabla F(y) +\nabla\log\varrho(t,y)} \dr\rho_t(x)\dr\rho_t(y)\dr t \,.
\end{multline*}

 It remains to study a last term: $\Pi_5$. And, we obtain by Lem.~\ref{lem:density} and~\ref{lem:continuityRho},
 \[
   \lim_{\varepsilon\to 0}\lim_{\delta\to0}   \Pi_5 = \int_{t_1}^{t_2}\!\!\!\int (\log \varrho(t,x) + F(x))\ps{\nabla \chi_R(x),v_t(x)}\dr\rho_t(x)\,.
 \]
 By Eq.~\eqref{eq:logrho} and ~\eqref{eq:bsupdensity}, 
 \[
 \sup_{t\in[t_1,t_2]}\int \norm{ (\log \varrho(t,x) + F(x)) \nabla\varrho(t,x)} \dr x < \infty \,.
 \]
Now, we remark that $\norm{\nabla \chi_R(x)} \le \frac C{\norm{x}}$.
Then, $$\sup_{t\in[t_1,t_2], x\in\bR^d}\norm{\nabla \chi_R(x)}\norm{ U(x,\rho_t) +\nabla F(x)} < \infty .$$
Consequently, by the two above equations, we can apply a dominated convergence theorem:
\[
    \lim_{R\to\infty}\lim_{\varepsilon\to 0}\lim_{\delta\to 0} \Pi_5 =0.
\]
Going back to Eq.~\eqref{eq:Pi}, we have shown
\[
    \KL{\rho_{t_2}}{\pi} - \KL{\rho_{t_1}}{\pi} = -\int_{t_1}^{t_2}\IS{\rho_t}{\pi} +\lambda \I{\rho_t}{\pi}  \dr t\,.
\]
\paragraph{Proof of Lem.~\ref{lem:continuityRho}}

 Using Eq.~(\ref{eq:continuityBis}) and integration by parts,
  \begin{multline*}
    \varrho^\varepsilon(t_2,x)-\varrho^\varepsilon(t_1,x) \\= -\int_{t_1}^{t_2}\!\! \int\ps{\nabla\eta_\varepsilon(x-y),b(y,\rho_s)} \dr \rho_s(y) \dr s  + \lambda \int_{t_1}^{t_2} \!\!\int \Delta\eta_\varepsilon(x-y) \dr\rho_s(y)\dr s \,.
  \end{multline*}
  Since $\rho \in \cP_2(\cC)$, $\sup_{t\in[t_1,t_2]} \|b(y,\rho_t)\| \le C(1+\norm{y}) + C\int  \sup_{t\in[t_1,t_2]} \norm{x_t} \dr \rho(x)$.
  As a consequence, $\sup_{t\in[1,T]} \|b(y,\rho_t)\| \le  C(1+\norm{y}) \,.$
Along with the observation that, for any fixed $\varepsilon$,  $\nabla \eta_\varepsilon$ and $\Delta \eta_\varepsilon$ are bounded,
it follows that $t\mapsto \varrho^\varepsilon(t,x)$ is Lipschitz continuous on $[t_1,t_2]$, and that its derivative almost everywhere is given by:
 $\6_t\varrho^\varepsilon(t,x) =\int (\ps{\nabla\eta_\varepsilon(x-y),b(y,\rho_t)}  +\lambda  \Delta\eta_\varepsilon(x-y)) \dr\rho_t(y)$. 
Thus, there exists a constant $C_{\varepsilon}>0$, such that:
\[
\sup_{t\in[t_1,t_2],x\in\bR^d}\6_t \varrho^{\varepsilon} (t,x) \le C_{\varepsilon}\,.
\]
$t\mapsto \nabla \varrho^\varepsilon(t,x)$ is also absolutely continuous by the same reasoning.

 \subsection{Proof of Prop.~\ref{prop:contraction}}
First, we introduce the Talagrand inequality $T_2$.
\begin{definition}
  The distribution $\pi$ satisfies the Talagrand inequality $T_2$, if there exists $\alpha >0$ such that for every $\mu\in\cP_2(\bR^d)$
  \[
    W_2(\mu,\pi) \le \sqrt{\frac 2\alpha \KL{\mu}{\pi}}\,.
  \]
\end{definition}
According to \cite[Th. 1]{otto2000generalization}, LSI implies $T_2$ with the same constant $\alpha$.

 In this subsection, we let Assumptions~\ref{hyp:stab},~\ref{hyp:holder} and Assumption~\ref{hyp:LSI} hold. Moreover, we assume $\lambda>0$.

 Let $\rho\in\sV_2$.
 By Prop.~\ref{prop:Lya} and Assumption~\ref{hyp:LSI}, we obtain
     \[
     \KL{\rho_{t_2}}{\pi} - \KL{\rho_{t_1}}{\pi} \le- 2\alpha\lambda \int_{t_1}^{t_2} \KL{\rho_t}{\pi}\dr t \,,
     \]
     for every $t_2>t_1>0$. By Grönwall's lemma, we obtain $\KL{\rho_{t_2}}{\pi} \le \ex^{-2\alpha\lambda(t_2-t_1)} \KL{\rho_{t_1}}{\pi}$. Using the Talagrand inequality $T_2$, we obtain 
   \[
   W_2(\rho_{t_2},\pi) \le \sqrt{\frac 2\alpha \KL{\rho_{t_1}}{\pi}} \ex^{-\alpha\lambda (t_2-t_1)}W_2(\rho_{t_1}, \pi)\,,
   \]
   for every $t_2>t_1>0$. Using Eq.~\eqref{eq:KLbound}, the proof is finished.
\section{Proof of convergence results}
In this section, we let Assumptions~\ref{hyp:algo},~\ref{hyp:stab}, and~\ref{hyp:holder} hold. Moreover, we assume $\lambda>0$.

First, we show the stronger ergodic convegergence result: 
\begin{prop}\label{prop:ergo} For every sequence $(\varphi_n,\psi_n)\to (\infty,\infty)$, we obtain
\[
   \lim_{n\to\infty} \bP\p{\frac{\sum_{i\in[\psi_n]}\gamma_i W_2(\mu^{\varphi_n}_i, \pi)}{\sum_{i\in[\psi_n]} \gamma_i} \ge \varepsilon} =0\,,
\]
for every $\varepsilon> 0$.
The latter still holds when we replace $W_2(\cdot,\cdot)$ by $W_2(\cdot, \cdot)^2$.

\end{prop}


\begin{proof}
By Lem.~\ref{lem:stab}, it is straightforward to check that  \cite[Cor.~1]{bianchi2024long} holds under Assumptions~\ref{hyp:algo} and~\ref{hyp:stab}. The proof consists in identifying the Birkhoff center $\text{BC}_2$, defined hereafter.

We define the translation $\Theta_t: x\in\cC\to x(t+\cdot)$.
We say that a point $\rho\in\sV_2$ is recurrent if there exists a sequence $(t_n)$ such that $\lim_{n\to\infty}(\Theta_{t_n})_\# \rho=\rho $. The Birkhoff center $\text{BC}_2$ is the closure of all recurrent points.

Let $\Lambda\subset \sV_2$. Let $\cF:\sV_2 \to \bR$ be a l.s.c. function such that $t\mapsto \cF((\Theta_t)_\# \rho)$ is strictly decreasing when $\rho\notin \Lambda$ and constant when $\rho\in\sV_2$. 
We say that a function $\cF$ defined as above is a Lyapunov function for a set $\Lambda$.

\begin{lemma}\label{lem:rec}
    Let $\cF$ be a Lyapunov function for a set $\Lambda$. Every recurrent points belongs to $\Lambda$.
\end{lemma}
\begin{proof}
    The limit $\ell:=\lim_{t\to\infty}\cF((\Theta_t)_\#\rho)$ is well-defined because
$\cF((\Theta_t)_\#\rho)$ is non increasing.  Consider a recurrent point $\rho\in
\sV_2$, say $\rho=\lim_{n} (\Theta_{t_n})_\#\rho$. Clearly $\cF(\rho)\geq
\cF((\Theta_{t_n})_\#\rho )\geq \ell$.  Moreover, by lower semi-continuity of $\cF$, $\ell
= \lim_n \cF((\Theta_{t_n})_\#\rho)\geq \cF(\rho)$. Therefore, $\ell$ is finite, and
$\cF(\rho)=\ell$.  This implies that $t\mapsto \cF((\Theta_t)_\#\rho)$ is constant. By
definition, this in turn implies $\rho\in \Lambda$, which concludes the proof.
\end{proof}

We define the l.s.c. function $\cF_\varepsilon:\rho\in\sV_2 \to \KL{\rho_\varepsilon}{\pi}$. By Prop.~\ref{prop:Lya}, this is a Lyapunov function for the set
\[
\Lambda_\varepsilon := \{ \rho\in\sV_2\,:\, \IS{\rho_t}{\pi} = \I{\rho}{\pi} =0,\, \forall t\ge \varepsilon \,a.e. \}\,.
\]
For $\mu\in\cP_2(\cC)$, $\I{\mu}{\pi} =0 $ implies $\mu=\pi$, and therefore $\KL{\mu}{\pi}=0$. Moreover, $t\mapsto \KL{\rho_t}{\pi}$ is constant for $t\ge \varepsilon$. Consequently,
\[
\Lambda_\varepsilon =\{\rho\in\sV_2\,:\, \rho_t = \pi,\, \forall t\ge\varepsilon \}.
\]
Let $\rho\in\sV_2$ a recurrent point, say $\lim_{n\to\infty}(\Theta_{t_n})_\#\rho = \rho$. 
By continuity of the projection $(\pi_0)_\#$, 
we obtain $\lim_{n\to\infty} \rho_{t_n} = \rho_0 =\pi$.

 Let $\rho\in\text{BC}_2$. It is a limit of recurrent points $\rho$ satisfying $\rho_0=\pi$. Hence, still by continuity of the mapping $(\pi_0)_\#$, $\rho_0 =\pi$. This finishes the proof of the fist claim of Prop~\ref{prop:ergo}.

 The second claim holds by redoing \cite[Prop.~\ref{coro:ergodic}]{bianchi2024long} with $W_2(\cdot,\cdot)^2$ instead of $W_2(\cdot,\cdot)^2$.
\end{proof}

Next, we state a stronger convergence result.
\begin{prop}
    \label{prop:pw}For every sequence $(\varphi_n,\psi_n)\to (\infty,\infty)$, we obtain
\[
   \lim_{n\to\infty} \bP\p{ W_2(\mu^{\varphi_n}_{\psi_n}, \pi) \ge \varepsilon} =0\,,
\]
for every $\varepsilon\ge  0$.
\end{prop}
\begin{proof}
By Prop.~\ref{prop:contraction}, we obtain
\begin{equation}\label{eq:compact}
\lim_{t\to\infty}\sup_{\rho\in\cK} W_2(\rho_t,\pi) =0\,,
\end{equation}
for every compact $\cK$ of $\cP_2(\cC)$.
Recall that the collection of random variables 
$\{m^n_t \}$ is tight in $\cP_2(\cC)$ by Lem.~\ref{lem:tightG}. Let
$(t_n,\varphi_n)$ be a sequence such that $(t_n,\varphi_n) \to_n
(\infty,\infty)$ and such that $(m^{\varphi_n}_{t_n})_n$ converges
in distribution to $M$. 
To prove Cor.~\ref{prop:pw}, it will be enough to show that 
\[
\forall \delta, \varepsilon > 0, \exists T > 0, \quad 
\limsup_n \bP\left(W_2\left((\pi_0)_\# m^{\varphi_n}_{t_n+T}, \pi\right) 
  \geq \delta\right) \leq \varepsilon. 
\] 
This shows indeed that 
\[
W_2\left( (\pi_0)_\# m^n_t,\pi \right)
   \xrightarrow[(t,n)\to (\infty,\infty)]{\bP} 0 , 
\]
and by taking $t = \tau_k$ and by recalling that $(\pi_0)_\# m^n_{\tau_k} = \mu_k^n$, 
we obtain our theorem. 

Fix $\delta$ and $\varepsilon$.  By the tightness of the family of random variables 
$\{m^n_t\}$, there exists a compact set $\cD\subset \cP_2(\cC)$ such that
$\bP(m^n_t \in \cD) \geq 1 - \varepsilon / 2$ for each couple $(t,n)$.
This implies that $M(\cD) \geq 1 - \varepsilon / 2$ by the Portmanteau theorem.
Since $\sV_2$ is closed by Lem.~\ref{lem:V2}, the set $\cK = \cD
\cap \sV_2$ is compact in $\cP_2(\cC)$, and by consequence, it is compact in
$\sV_2$ for the trace topology.  By the same proposition, $M(\sV_2) = 1$,
therefore, $M(\cK) \geq 1 - \varepsilon / 2$. 

Since $\cP_2(\cC)$ is Polish, we can apply Skorokhod's representation theorem
\cite[Th.~6.7]{billingsley2013convergence} to the sequence
$(m^{\varphi_n}_{t_n})$, yielding the existence of a probability space
$(\widetilde \Omega, \widetilde \cF, \widetilde \bP)$, a sequence of
$\cP_2(\cC)$--valued random variables $(\rho^n)$ on $\widetilde\Omega$ and
a $\cP_2(\cC)$--valued random variable $\rho^\infty$ on
$\widetilde\Omega$ such that $(\rho^{n})_\#\widetilde\bP =
(m^{\varphi_n}_{t_n})_\#\bP$, $( \rho^\infty)_\#\widetilde \bP = M$,
and $\rho ^n \to \rho^\infty$ pointwise on $\widetilde\Omega$.  Noting
that $(\pi_0)_\# m^{\varphi_n}_{t_n+T}$ and $\rho^n_T$ have the same probability
distribution as $\cP_2(\bR^d)$--valued random variables, we show that 
\begin{equation}
\label{eq:Ptilde}
 \exists T > 0, \quad 
  \limsup_n \widetilde\bP \left(W_2\left(\rho^{n}_T, \pi\right) 
  \geq \delta\right) \leq \varepsilon,
\end{equation}
to establish our theorem. Applying Eq.~\eqref{eq:compact} 
to the compact $\cK$, we set $T > 0$ in such a way that 
\[
\sup_{\rho \in \cK} W_2(\rho_T, \pi) \leq \delta / 2. 
\]
By the triangular inequality, we have 
\[
W_2\left(\rho^{n}_T, \pi\right) \leq 
  W_2\left( \rho^{n}_T, \rho^{\infty}_T\right) 
  + W_2\left(  \rho^{\infty}_T, \pi\right).
\]
The first term at the right hand side converges to zero for each
$\tilde\omega\in\widetilde\Omega$ by the continuity of the function $\rho
\mapsto \rho_T$, thus, this convergence takes place in probability.
We also know that for $\widetilde\bP$--almost all
$\tilde\omega\in\widetilde\Omega$, it holds that $\rho^\infty \in \sV_2$.
Thus, regarding the second term, we can write    
\[
\widetilde\bP\left( 
  W_2\left(\rho^{\infty}_T, \pi\right) \geq\delta 
  \right) \leq 
\widetilde\bP\left( \rho^{\infty} \not\in \cK \right) 
 + 
 \widetilde\bP\left( 
  \left( 
   W_2\left(\rho^{\infty}_T, \pi\right) \geq\delta 
  \right) \cap \left( \rho^{\infty} \in \cK \right) \right) . 
\] 
When $\rho^{\infty}\in \cK$, it holds that 
$W_2\left(\rho^{\infty}_T, \pi\right) \leq \delta / 2$,
thus, the second term at the right hand side of the last inequality is 
zero. The first term satisfies 
$\widetilde\bP\left( \rho^{\infty} \not\in \cK \right) = 
 1 - M(\cK) \leq \varepsilon / 2$, and the statement~\eqref{eq:Ptilde} 
follows. Cor.~\ref{prop:pw} is proven.
\end{proof}

\subsection{Proof of Th.~\ref{th:ergo}}

Instead of seeing $\bar{\mathscr L}^n$ as set of random variable on $\cP_2(\bR^d)$, we see it as a set of measures in $\cP(\cP_2(\bR^d))$. We denote such a set as $\bar \cL^n$. 

Let $\varepsilon>0$. By contradiction, there exists $\delta>0$, a subsequence $\varphi_n \to \infty$ and a sequence of measures $\nu^n\in\bar \cL^{\varphi_n}$ satisfying 
\[
 \int \indicatrice_{W_2(\mu,\pi)>\varepsilon} \dr \nu^n(\mu) \ge \delta\,.
\]
As shown in the proof of Th.~\ref{th:tight}, the sequence of random variable $(\bar \mu_k^n\,:\, k,n\in\bN^*)$ is tight. Hence, there exists a measure $\nu^\infty\in\cP_2(\bR^d)$ such that $(\nu^n)$ converges to $\nu^\infty$ along a subsequence. To keep the notations simple, we say that $\nu^n\to\nu^\infty$.
Since, $\mu\in\cP_2(\bR^d)\mapsto\indicatrice_{W_2(\mu,\pi)}$ is continuous bounded, we obtain 
\[
 \int \indicatrice_{W_2(\mu,\pi)>\varepsilon} \dr \nu^\infty(\mu) \ge \delta\,.
\]
Let \((\psi_k^n)_k\) be a sequence diverging to \(\infty\) such that \(\bar{\mu}_{\psi_k^n}^n \to_k\nu^n\), for every \(n \in \bN^*\).

Let $\varepsilon'>0$, there exists $n_0$ such that,
\[
\abs{ \int \indicatrice_{W_2(\mu,\pi)>\varepsilon} \dr \nu^\infty(\mu)-  \int \indicatrice_{W_2(\mu,\pi)>\varepsilon} \dr \nu^{n_0}(\mu)} \le \frac {\varepsilon'} 2\,.
\]
Moreover, there exists \(k_0\) such that 
\[
\abs{ \bP( W_2(\bar \mu_{\psi_{k_0}^{n_0}}^{n_0},\pi)>\varepsilon)  -  \int \indicatrice_{W_2(\mu,\pi)>\varepsilon} \dr \nu^{n_0}(\mu)} \le \frac {\varepsilon'} 2\,.
\]
Consequently, there exists a subsequence
\((\tilde{\varphi}_n, \tilde{\psi}_n) \to (\infty, \infty)\) such that
\[
\lim_{n\to\infty} \bP(W_2(\bar \mu_{\tilde\psi_n}^{\tilde\varphi_n},\pi)\ge \varepsilon ) = \int \indicatrice_{W_2(\mu,\pi)>\varepsilon} \dr \nu^{\infty}(\mu) \ge \delta\,.
\]
By Jensen's inequality, we obtain
\[
W_2(\bar \mu_{\tilde\psi_n}^{\tilde\varphi_n},\pi)^2\le \frac{\sum_{k\in[\tilde \psi_n]} \gamma_k W_2(\mu^{\tilde \varphi_n}_k,\pi)^2}{\sum_{k\in[\tilde \psi_n]} \gamma_k}\,.
\]
Consequently,
\[
\lim_{n\to\infty}\bP\p{\frac{\sum_{k\in[\tilde \psi_n]} \gamma_k W_2(\mu^{\tilde \varphi_n}_k,\pi)^2}{\sum_{k\in[\tilde \psi_n]} \gamma_k}\ge \varepsilon^2 } \ge \delta\,.
\]
The latter contradicts the second claim of Prop.~\ref{prop:ergo}. Thus, the proof is finished.


\subsection{Proof of Th.~\ref{th:nonergo}}
This is the same proof as Th.~\ref{th:ergo}. But this time, we use Prop.~\ref{prop:pw}.

\subsection{Proof of Cor.~\ref{coro:ergodic}}
By contradiction, assume that there exists $\delta>0$ and a subsequence $\varphi_n$, such that for every $n$, $\limsup_{k\to\infty} \bP(W_2( \bar\mu_k^{\varphi_n},\pi) \geq \varepsilon)>\delta$. Assume $\varphi_n=n$ to simplify the notations.
For any $n$, this implies that one can extract a subsequence, say $(\psi_k^n:k\in \bN)$, such that for every $k$, $\bP(W_2( \bar\mu_{\psi_k^n}^{n},\pi) \geq \varepsilon)>\delta/2$.
By Th.~\ref{th:tight}, the sequence $(\bar\mu_{\psi_k^n}^{n}:k\in \bN)$ is tight, so that there exists $\nu^n\in \bar{\mathscr{L}}^n$, such that
$\bar\mu_{\psi_k^n}^{n}$ converges in distribution to $\nu^n$ as $k\to\infty$, along some subsequence which we still denote by $\psi_k^n$ to keep the notations simple. By the Portmanteau theorem, 
 \begin{equation}
     \label{eq:limsuppsi}
     \limsup_{k\to\infty} \bP(W_2( \bar\mu_{\psi_k^n}^{n},\pi) \geq \varepsilon) \leq \bP(W_2( \nu^n,\pi) \geq \varepsilon)\,.
 \end{equation}
By Th.~\ref{th:ergo}, $\nu^n$ converges in probability to $\pi$ in $\cP_2(\bR^d)$ as $n\to\infty$. Therefore, $\bP(W_2( \nu^n,\pi) \geq \varepsilon)<\delta/3$ for all $n$ large enough. Using Eq.~(\ref{eq:limsuppsi}), it follows that $\bP(W_2( \bar\mu_{\psi_k^n}^{n},\pi) \geq \varepsilon)<\delta/2$ along some subsequence, hence a contradiction. This proves the first point. The second point follows the same arguments.

\end{document}